\documentclass{article} 
\usepackage{iclr2027_conference,times}

\usepackage{amsmath,amsfonts,bm}

\def\eqref#1{equation~\ref{#1}}

\def\1{\bm{1}}

\DeclareMathAlphabet{\mathsfit}{\encodingdefault}{\sfdefault}{m}{sl}
\SetMathAlphabet{\mathsfit}{bold}{\encodingdefault}{\sfdefault}{bx}{n}

\usepackage{hyperref}
\usepackage{url}
\usepackage{amsmath,amssymb,amsthm}
\usepackage{graphicx}
\usepackage{booktabs}
\newtheorem{theorem}{Theorem}[section]

\newtheorem{lemma}[theorem]{Lemma}
\theoremstyle{definition}
\newtheorem{definition}[theorem]{Definition}
\theoremstyle{plain}
\usepackage{xcolor}
\definecolor{revisionblue}{RGB}{0,70,140}

\title{Beyond Arrow: From Impossibility to Possibility in Multi-Criteria Benchmarking}

\author{%
  \textbf{Polina~Gordienko}\textnormal{\textsuperscript{1,2,3,4}}
  \qquad
  \textbf{Christoph~Jansen}\textnormal{\textsuperscript{5}}
  \\[0.5em]
  \textbf{Julian~Rodemann}\textnormal{\textsuperscript{1,6}}
  \qquad
  \textbf{Georg~Schollmeyer}\textnormal{\textsuperscript{1}}
  \\[1em]
  \textnormal{\textsuperscript{1} Department of Statistics, LMU Munich} \\
  \textnormal{\textsuperscript{2} Munich Center for Machine Learning (MCML)} \\
  \textnormal{\textsuperscript{3} relAI -- Konrad Zuse School of Excellence in Reliable AI} \\
  \textnormal{\textsuperscript{4} University of Oxford} \\
  \textnormal{\textsuperscript{5} School of Computing \& Communications,} \\
  \textnormal{\phantom{\textsuperscript{5} }Lancaster University Leipzig} \\
  \textnormal{\textsuperscript{6} CISPA Helmholtz Center for Information Security}
  \\[0.5em]
  \textnormal{\texttt{Polina.Gordienko@stat.uni-muenchen.de}}%
}

\iclrfinalcopy
\begin{document}

\maketitle
\lhead{Preprint}

\begin{abstract}

Modern benchmarks such as HELM MMLU account for multiple metrics like accuracy, robustness and efficiency. When trying to turn these metrics into a single ranking, natural aggregation procedures can produce cyclic pairwise comparisons or rankings that are unstable to changes in the model set. We formalize this aggregation as a social choice problem where each metric induces a preference ranking over models on each dataset, and a benchmark operator aggregates these votes across metrics. While prior work has focused on Arrow's impossibility result, we identify sufficient conditions under which the impossibility disappears and meaningful multi-criteria benchmarking becomes possible. In particular, we deal with three restrictions on the combinations of rankings and show that on single-peaked, group-separable and distance-restricted preferences, the benchmark operator allows for the construction of well-behaved rankings of the involved models. Empirically, we investigate several modern benchmark suites like HELM MMLU and verify which structural conditions are fulfilled on which benchmark problems.
  
\end{abstract}

\section{Introduction}\label{intro}

Benchmarks are fundamental to progress in modern machine learning, for they provide standardized frameworks for model evaluation and comparison \citep{hardtrecht2022patterns}. Within the past decades, evaluation methods in the field have expanded from single-task benchmarks such as leaderboards based on the ImageNet dataset to multi-task suites including GLUE \citep{wang2018}, SuperGLUE \citep{wang2019}, MMLU \citep{Hendrycks2021} and BIG-Bench \citep{Srivastava2022BeyondTI}. More recently, holistic evaluation frameworks like HELM \citep{liang2023} have been introduced that combine hundreds of tasks to assess language models across a broad range of scenarios and along different metrics that go beyond accuracy. Yet, as these benchmark suites have grown more complex, a central challenge has emerged: how should we aggregate performance across multiple evaluation criteria, each measured over several datasets, into one meaningful ranking of models?

The prevailing approach in benchmarking is based on the ``unspoken utilitarian principles'' \citep{Rofin_2023}, ranking systems by the arithmetic mean of their scores across task-specific metrics \citep{wang2018, wang2019, Hendrycks2021}. Nonetheless, this aggregation procedure has been increasingly viewed as inadequate, for instance, due to its sole focus on the value of model predictions and disregard for the cost (e.g., model size, training time) of those predictions \citep{ethayarajh20}. Moreover, mean aggregation can be easily dominated by performance on a few outlier tasks \citep{agarwal21}. In general, averaging methods can mislead when samples vary in difficulty, resulting in inflated model performance and unreliable rankings \citep{mishra2021}. The problem is amplified when benchmarks include ordinal criteria such as interpretability, since mean aggregation presupposes a meaningful scale of differences that such metrics do not provide. Further challenges such as task selection bias \citep{dehghani21}, saturation of benchmarks over time \citep{kiela2021} as well as limitations of dynamic benchmarks \citep{shirali2023theorydynamicbenchmarks} have been identified. However, it has also been argued that a focus on failures merely incentivizes the development of models that sidestep those mistakes by shifting errors elsewhere, creating only the appearance of progress \citep{bowmandahl21}. Instead of focusing on limitations, our work identifies conditions under which meaningful multi-criteria benchmarking becomes possible. Rather than asking ``what goes wrong?'', we ask ``under what structural conditions can benchmarking succeed?''.

Social choice theory -- the ``study of collective decision procedures and mechanisms'' \citep{sep-social-choice} -- offers a precise vocabulary to assess aggregation processes. While Arrow's theorem (\citeyear{arrow}) establishes that no aggregation function can satisfy all desirable properties simultaneously over the \textit{universal domain} (i.e., the function works for every possible pattern of preferences), the social choice literature has long recognized that \textit{domain restrictions} -- natural constraints on the inputs of the aggregation rule -- can restore the possibility of meaningful aggregation \citep{ black48, inada64, sen, DIETRICH2010512, PUPPE2024426}.  In the context of benchmarking, this raises a fundamental question that has so far received surprisingly little attention: \textbf{Do multi-criteria benchmarks naturally exhibit structure in the domain that restores possibility?} In contrast to prior work focusing on impossibility \citep{zhang2024inherent}, we deal with sufficient conditions under which well-behaved multi-criteria benchmarking becomes possible.  

\textbf{Contributions.} We make the following three contributions.

1. We formalize multi-criteria, multi-task benchmarking as a preference aggregation problem in two steps: aggregation within each dataset, where metrics (e.g., accuracy, fairness, efficiency) act as voters and models as alternatives, and subsequent aggregation of the resulting rankings across datasets. 

2. We show that if metric rankings satisfy \textit{single-peakedness}, \textit{group separability} or \textit{distance-restrictedness}, then pairwise majority yields well-behaved leaderboard rankings. Intuitively, single-peakedness means that there exists a one-dimensional ordering of models such that each metric has a single ``sweet spot", and moving away from it along the ordering makes models consistently less preferred by that metric (see Figure~\ref{single-peaked}). When it holds empirically, single-peakedness gives an ordinal counterpart to recent evidence that benchmark performance often has a low-dimensional, sometimes nearly one-dimensional, structure \citep{ruan2024observational, hardt2025emerging}.

3. We turn these conditions into a diagnostic protocol that can be tested on benchmarks and empirically examine when the conditions are satisfied. We test these structural assumptions on HELM MMLU across all 83 model sets and all 42 metric subsets of sizes 3-6, under two deterministic tie-breaking rules (6972 evaluation settings) and ten repetitions of random tie-breaking. We compare with matched uniformly random profiles and conduct further experiments on PMLB and OpenML.

\begin{figure}[t]
\centering
\begin{minipage}[t]{0.49\textwidth}
\centering
\includegraphics[width=\linewidth]{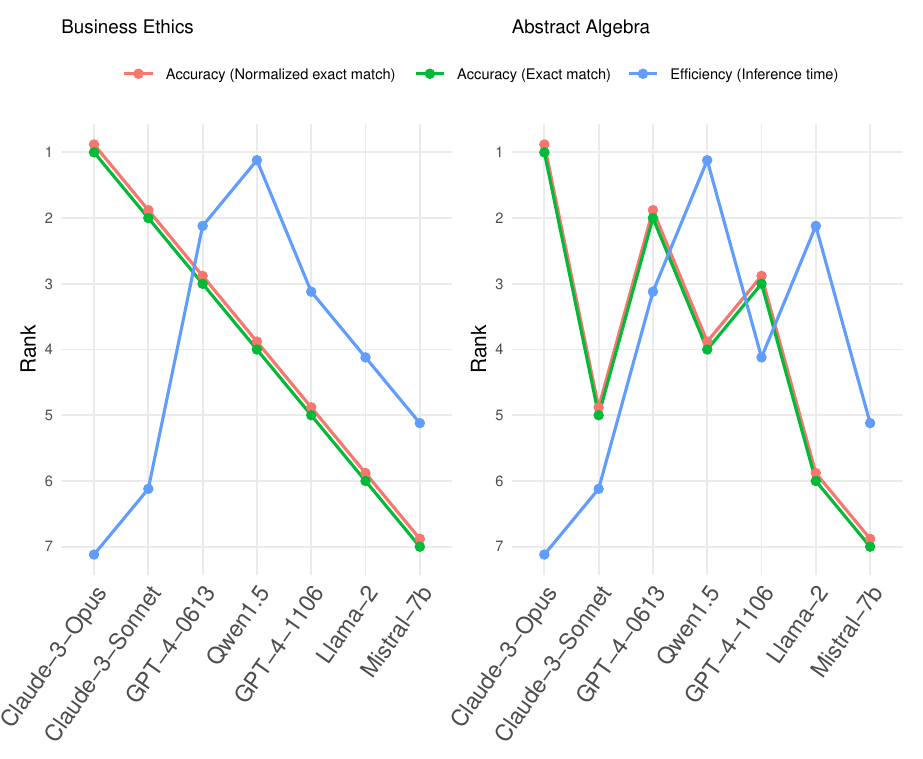}
\end{minipage}
\hfill
\begin{minipage}[t]{0.49\textwidth}
\centering
\includegraphics[width=\linewidth]{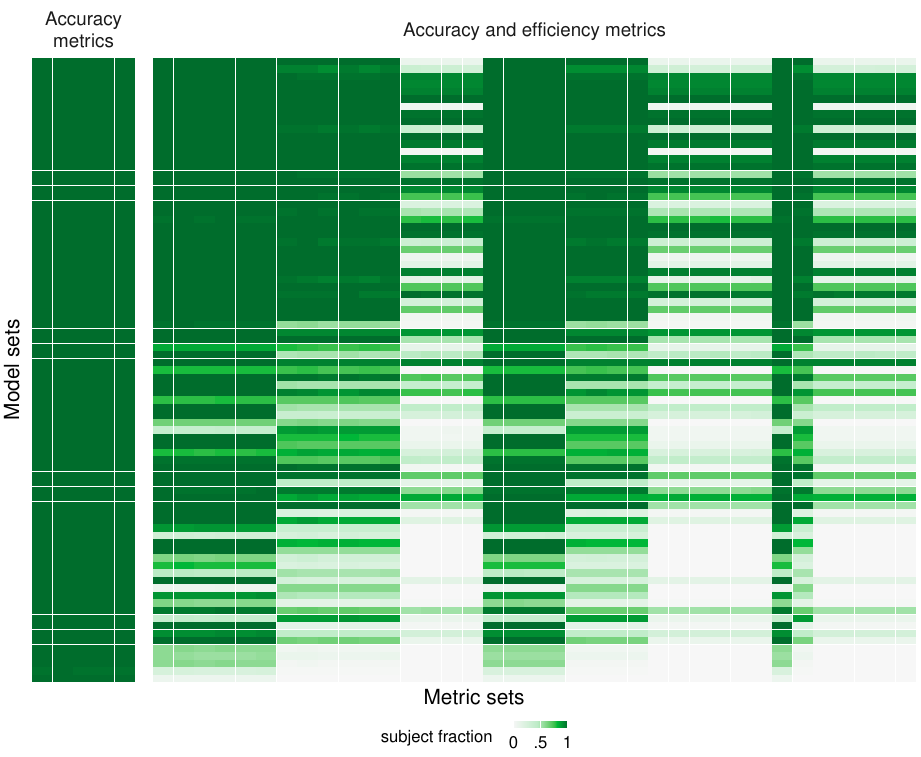}
\end{minipage}
\caption{Single-peaked structure in HELM MMLU. Left: for each MMLU subject dataset, each of the three HELM metrics induces a ranking of models; we study when these rankings are consistent with a single shared ordering of models (x-axis) so that each metric has one ``sweet spot" (\textit{peak}) along that ordering. \textit{Business Ethics} satisfies this structure, \textit{Abstract Algebra} does not. Right: each tile is one combination of model set and metric set; color shows the fraction of $57$ MMLU subjects on which metric rankings are single-peaked. In $1426$ out of $3486$ evaluation settings, single-peakedness holds for all subjects.}
\label{single-peaked}
\end{figure}

\section{Related Work}\label{lit}

A growing body of literature has turned to social choice theory to formalize the aggregation problem in benchmarks. Early contributions include \citet{ehl2012} who frame benchmarking as consensus over dataset-level preference relations as well as \citet{mptbw2015} who study benchmarking of optimization algorithms as a consensus-ranking problem over many test functions. \citet{colombo2022bestsystemsnewperspectives} propose an aggregation method based on Kemeny consensus, while \citet{himmi2023robustnlpevaluationhandling} develop a ranking approach using the Borda count for managing missing scores in benchmarks. \citet{Rofin_2023} introduce VOTE'N'RANK, a framework applying eight social choice procedures onto multi-task benchmarks. \citet{zhang2024inherent} use Arrow's theorem to analyze the aggregation problem in multi-task benchmarks, highlighting a trade-off between diversity and stability. \citet{lanctot2025evaluatingagentsusingsocial} employ principles from social choice and game theory for building a framework to evaluate general agents. \citet{gordienko2026} use concepts from computational social choice to study the robustness of benchmark rankings to benchmark-specific training. Complementary work \citep{JMLR:v7:demsar06a,JMLRBenavoli,JMLR:v24:22-0902,gsd,longjohn2025statisticaluncertaintyquantificationaggregate} proposes statistically sound methods for aggregating performance across multiple tasks.

\section{Pathologies in Multi-Criteria Benchmarking}\label{motiv}
We deal with multi-criteria benchmarking that aims to evaluate models along different criteria/metrics, each measured over multiple datasets. The central challenge is how to obtain a well-behaved, overall ranking of models from criteria across qualitatively different dimensions. But what does ``well-behaved" mean? We argue that two intuitive properties are \textit{transitivity} and \textit{stability}.

\textbf{Transitivity.} One plausible constraint is that the overall ranking of models should not contradict itself. For three models $A,B,C$, it may happen that metrics collectively prefer $A$ to $B$, $B$ to $C$, and yet $C$ to $A$. Even if each metric induces a sound ordering of models, there does not necessarily exist a transitive aggregated ranking over all metrics. We illustrate this tension on MMLU under HELM, which evaluates 23 language models across 57 subjects and records multiple metrics per subject \citep{Hendrycks2021, liang2023}. Each metric generates a ranking of models for each subject. A natural aggregation method that treats all metrics equally is pairwise majority. We consider each metric as one \textit{vote} and declare that $A$ beats $B$ if more metrics prefer $A$ to $B$ than prefer $B$ to $A$. Table~\ref{table:motiv_examples} (left) illustrates that these pairwise comparisons may result in cyclic majority rankings such as GPT-4 $\succ$ Qwen1.5 $\succ$ GPT-3.5 $\succ$ GPT-4 (where $\succ$ is short for ``is preferred to''). The majority ranking is not well-behaved: even this basic aggregation method fails at delivering a coherent ranking of models across three metrics. Since the pairwise majority relation is cyclic on a subset of models, there is no transitive ranking over the full model set that agrees with all majority comparisons at once. This phenomenon is not rare. Searching over triples of accuracy and efficiency metrics in HELM MMLU, we find such cycles in 51 of 57 subjects (Appendix~\ref{app:cycles}).

\begin{table}[h]
\caption{Two pathologies in multi-criteria benchmarking. Left: on \textit{Formal Logic}, columns list the three models from best to worst under that metric, with values in parentheses; for \textit{Accuracy} higher means better, for \textit{Time} and \textit{Length} lower means better. Pairwise majority yields GPT-4 $\succ$ Qwen1.5 $\succ$ GPT-3.5 $\succ$ GPT-4. Right: on \textit{High School World History}, rankings are computed across \textit{Accuracy}, \textit{Time} and \textit{Length} by average rank before and after Llama-2 is added. The relative order of Claude-3-Opus and GPT-3.5 flips. Exact model identifiers can be found in Appendix \ref{app:expmotiv}.}
\label{table:motiv_examples}
\centering
\small
\setlength{\tabcolsep}{3.5pt}
\begin{minipage}[t]{0.49\textwidth}
\begin{center}
\begin{tabular}{@{}lll@{}}
\toprule
\textbf{Accuracy} & \textbf{Time (s)} & \textbf{Length (bytes)} \\
\midrule
GPT-4 (0.65)   & Qwen1.5 (0.32) & GPT-3.5 (1.00) \\
Qwen1.5 (0.49) & GPT-3.5 (0.41) & GPT-4 (1.17) \\
GPT-3.5 (0.40) & GPT-4 (0.49)   & Qwen1.5 (2.00) \\
\bottomrule
\end{tabular}
\end{center}
\end{minipage}\hfill
\begin{minipage}[t]{0.49\textwidth}
\begin{center}
\begin{tabular}{@{}ll@{}}
\toprule
\textbf{Before} & \textbf{After} \\
\midrule
1.\ GPT-4-0613 & 1.\ GPT-4-0613 \\
2.\ GPT-4-1106 & 2.\ GPT-4-1106 \\
3.\ Claude-3-Opus & 3.\ Qwen1.5 \\
4.\ Qwen1.5 & 4.\ GPT-3.5 \\
5.\ GPT-3.5 & 5.\ Claude-3-Opus \\
\vdots & \vdots \\
15.\ Google-Text-Bison & 15.\ Google-Text-Bison \\
 & 16.\ Llama-2 \\
\bottomrule
\end{tabular}
\end{center}
\end{minipage}
\end{table}

\textbf{Stability.} We can avoid cycles in benchmark rankings by using another aggregation method, for instance, any rule based on combining ranks of models across different datasets/tasks or metrics. One example is Borda score \citep{borda1781}, which is being increasingly used in benchmarks such as MTEB \citep{Chung2025MaintainingMT} and yields identical rankings as mean win rate, a widely used benchmarking rule \citep{liang2023, Lee2023HolisticEO, zhang2024inherent, hardt2025emerging}. Then, however, we face another problem: benchmarks may become unstable to irrelevant changes in the model set. Stability means that adding or removing a model $C$ should not flip the ranking between two existing models $A$ and $B$. Table~\ref{table:motiv_examples} (right) demonstrates the violation of stability on one dataset in HELM MMLU. We consider 15 models with the highest accuracy on this subject among those with complete measurements for the three metrics, and then add one additional model. We compute the overall rankings of models before and after the change across metrics by averaging the metrics' ranks of models (Borda score over ranks). Note that the new model, Llama-2, is worse on \textit{Accuracy} compared to all existing 15 models, yet relatively efficient with a lower \textit{Time} than Claude-3-Opus but a higher \textit{Time} than GPT-3.5. Crucially, the addition of Llama-2 changes the rankings of models near the top of the leaderboard. We find such flips in 44 out of 57 MMLU subjects (see Appendix \ref{app:instability}).

By no means do these failures imply the impossibility of well-behaved multi-criteria benchmarking. Rather, these examples tell us what must be true of the rankings induced by each metric if transitive and stable leaderboard rankings are to be possible. For the aforementioned pathologies arise only in the \textit{unstructured} combinations of rankings. In this paper, we show meaningful aggregation can be achieved and analyze which structures are sufficient for transitive and stable leaderboards.

\section{Benchmarking as a Social Choice Problem}\label{benchmarking}
Social choice theory studies the aggregation of individual inputs -- preferences, judgements or probabilities -- into an output, and investigates axiomatic properties that such collective decision procedures can or cannot satisfy. In this section, we frame benchmarking as a social choice problem.

\subsection{Notation}\label{setup}
\textbf{Multi-criteria benchmarking.} Let $\mathcal{D}$ be some fixed universe of instances/datasets/tasks and let $\mathcal{A}$ be some fixed and finite set of models with $k:=|\mathcal{A}|\in \mathbb{N}$. Let $(\phi_i)_{i \in [n]}$ denote a family of metrics/criteria $\phi_i: \mathcal{A}\times \mathcal{D} \to [0,1],$ for some $n \in \mathbb{N}$, where $[n]:=\{1, \dots,n\}$. Let $\Phi=(\phi_1 , \dots , \phi_n): \mathcal{A}\times \mathcal{D} \to [0,1]^n$ be a multivariate metric.\footnote{Note that our framework is more general than a mere multi-criteria benchmarking problem: any problem that produces evaluations over a finite set of alternatives can be cast in terms of $\mathcal{D}$, $\mathcal{A}$ and $\Phi$, e.g., one may treat datasets or tasks as voters \citep{zhang2024inherent} or consider test functions as voters in optimizer benchmarking \citep{rodemann2024partialrankingsoptimizers}.} The goal of multi-criteria benchmarking is to search for a model from $\mathcal{A}$ that best balances the metrics $\phi_1 , \dots , \phi_n$ across the various datasets $D \in \mathcal{D}$.

\textbf{Preference aggregation problem.} Let pref$(\mathcal{A})$ be the set of all complete and transitive \textit{preference relations} on $\mathcal{A}$. Each fixed metric $ i \in [n]$ and dataset $D \in \mathcal{D}$ induce a preference relation $R_{D,i} \in$ pref$(\mathcal{A})$, defined by setting $A \succeq_{R_{D,i}} A^{'} :\Leftrightarrow \phi_i(A,D) \geq \phi_i(A^{'},D).\footnote{In general, we assume that larger values of metric $\phi$ mean better performance on that metric. Otherwise, e.g, for metrics \textit{Time} and \textit{Length}, we flip the sign of the values.}$ Hence, $R_{D,i}$ corresponds to the ranking of models in $\mathcal{A}$ induced by the scores of the $i$th performance metric, evaluated on instance $D$. We denote by $P_{D,i}$ the strict part of $R_{D,i}$ defined by $A \succ_{P_{D,i}} A^{'} :\Leftrightarrow \phi_i(A,D) > \phi_i(A^{'},D)$. Since metric scores can tie, $R_{D,i}$ may include indifference between models; we use $P_{D,i}$ whenever strict comparisons are necessary. We obtain $P_{D,i}$ from $R_{D,i}$ by breaking ties according to a fixed deterministic rule. Furthermore, let prof$(\mathcal{A},\mathcal{D},\Phi)$ denote the set of \textit{preference profiles} induced by $\mathcal{D}$, i.e., the set $\text{prof}(\mathcal{A},\mathcal{D},\Phi)=\{R_D:=(R_{D,1}, \dots , R_{D,n})| D \in \mathcal{D}\}.$ Each profile $R_D$ arises by collecting the preference rankings induced by metrics $\phi_1 , \dots ,\phi_n$ evaluated on instance $D$ in an $n$-tuple. Let $B:\text{prof}(\mathcal{A},\mathcal{D},\Phi)\to$ pref$(\mathcal{A})$ be a \textit{benchmark operator} -- a map returning a preference order among the models of interest for every profile of such rankings that is inserted to it. Importantly, the domain of the operator $B$ depends on each of the objects $\mathcal{A},\mathcal{D},\Phi$.

In HELM MMLU, $\mathcal{D}$ corresponds to the set of 57 MMLU subject datasets, while each $D \in \mathcal{D}$ is a single subject (e.g. \textit{Formal Logic}). The set of evaluated language models is denoted by $\mathcal{A}$, while $\Phi$ is the collection of all reported HELM metrics, with each $\phi_i$ corresponding to the metric $i$. The relation $R_{\textit{Formal Logic, Time}}$ denotes the ranking of models according to metric \textit{Time} on subject \textit{Formal Logic}, while the profile $R_{\textit{Formal Logic}}$ includes the rankings induced by all metrics on this subject. The benchmark's overall ranking of models on that subject is given by $B(R_{\textit{Formal Logic}})$. 

There are many kinds of benchmark operators, including pairwise majority and mean win rate mentioned in Section~\ref{motiv}. Pairwise majority aggregation is the most reasonable local aggregation rule, though its downside -- the possibility of cyclic majority rankings -- is well-known. We show that there are many realistic benchmarking scenarios where pairwise majority not only produces a transitive overall ranking of models but also satisfies further desirable properties that we introduce in the next section. This makes it, under the structural assumptions we study, more attractive for benchmarking than rules like mean win rate.\footnote{In benchmarking, rules based on combining ranks of models across tasks (e.g. mean win rate) are common. However, these procedures are unstable to changes in model sets (e.g., see Table~\ref{table:motiv_examples}) and violate Arrow's \textit{IIA} \citep{JMLR:v17:benavoli16a}.} For any models $A,A^{'} \in \mathcal{A}$ and any tasks $D \in \mathcal{D}$, $M_D(A,A^{'}) := |\{i \in [n] : A \succeq_{R_{D,i}} A^{'}\}|$ counts how many metrics rank $A$ at least as high as $A'$ and
$A \succeq_M^D A^{'} \Leftrightarrow M_D(A,A^{'}) \ge M_D(A^{'},A).$

\subsection{Social Choice: From Impossibility to Possibility}
\textbf{Arrow's theorem.} Social choice is known for its impossibility results, most prominently Arrow's theorem (\citeyear{arrow}). Yet, the problem behind impossibilities in voting was recognized much earlier, by Condorcet (\citeyear{condorcet1785}) who formalized a situation where three voters rank the alternatives $x$, $y$, and $z$ and use majority voting on each pair of alternatives. Even though each voter's preference is rational, majority prefers $x$ over $y$, $y$ over $z$, and $z$ over $x$, yielding a \textit{Condorcet cycle} and an irrational collective ranking. The cycle in Table~\ref{table:motiv_examples} corresponds to the Condorcet paradox. Arrow (\citeyear{arrow}) generalized it by showing that, given a set of individual preferences over three or more alternatives, there exists no aggregation procedure that satisfies a
few reasonable assumptions (called \textit{axioms}) concerning the autonomy of voters and the rationality of their preferences \citep{sep-arrows-theorem}.

\begin{theorem} \cite{arrow}
    Suppose $k\geq 3$ and $n\geq 2$. Let $F$ be an operator that maps each profile from $($pref$(\mathcal{A}))^n$ (\textbf{Universal Domain}) to a preference relation in pref$(\mathcal{A})$ (\textbf{Social Ordering}). For profiles $R=(R_{1}, \dots , R_{n})$ and $R^{'}=(R^{'}_{1}, \dots , R^{'}_{n})$, let $P_{i}, P^{'}_{i}$ be the strict parts of $R_{i}$ and $P, P^{'}$ the strict parts of $F(R), F(R)^{'}$. Then $F$ cannot satisfy the following conditions simultaneously:
    
\textbf{Non-Dictatorship:}  There is no $i \in [n]$ such that $\forall R$ and $\forall A, A^{'} \in \mathcal{A}$, if $A \succ_{P_{i}} A^{'}$, then $A \succ_{P} A^{'}$.

\textbf{Weak Pareto:} $\forall R$ and $\forall A, A^{'} \in \mathcal{A}$, if $ \ \forall i \in [n] \ \ A \succ_{P_{i}} A^{'}$, then $A \succ_{P} A^{'}$.

\textbf{Independence of Irrelevant Alternatives (IIA):} $\forall R, R^{'}$ and $ \forall A, A^{'} \in \mathcal{A}$, if $ \ \forall i \in [n] \ \ A \succeq_{P_{i}} A^{'}  \Leftrightarrow A \succeq_{P^{'}_{i}} A^{'}$, then $A \succeq_{P} A^{'}  \Leftrightarrow A \succeq_{P^{'}} A^{'}$.\footnote{Note that the operator $F$ generally differs from the operator $B$ as defined earlier, since in most cases the benchmark suite $B$ does not induce the full domain of preference profiles $($pref$(\mathcal{A}))^n$. This is why Arrow's impossibility theorem cannot be explicitly applied to the benchmark operator $B$.}

\end{theorem} 

\textit{Social Ordering} requires the benchmark to output a coherent ranking, i.e. a complete and transitive binary relation on $\mathcal{A}$. \textit{Universal Domain} demands that the aggregation rule admits any logically possible combination of metric rankings over $\mathcal{A}$. This implies that we are not allowed to make any assumptions about metrics' behavior or relationships between metrics in advance. In benchmarking, however, \textit{Universal Domain} is not a reasonable assumption. For metrics do exhibit systematic behavior, and trade-offs across families of metrics are well-known (e.g., \cite{kaplan2020,Ang2022CharacterizingTE}). \textit{Weak Pareto} asks that if every metric ranks model $A$ over model $A^{'}$, then the collective ranking across metrics must prefer $A$ to $A^{'}$. \textit{Non-Dictatorship} rules out any single metric that always determines the overall preferential ranking across all metrics. Finally, the theorem's most demanding condition of \textit{IIA} states that the social comparison of models $A$ and $A^{'}$ depends only on how each metric compares the two models and not on how any metric ranks other models. Thus, the intransitivity in Table~\ref{table:motiv_examples} is a violation of \textit{Social Ordering}, while the instability to changes in model set implies a violation of the benchmark analogue of \textit{IIA}.

\textbf{Restricted preference domains.} 
A longstanding literature \citep{black48,sen, inada69, DIETRICH2010512, elkind25} investigates \textit{possibilities}, under which majority cycles can be avoided. The idea is to relax or violate one of Arrow's axioms in a plausible manner under which the remaining conditions can become jointly satisfiable. One ``escape route" from impossibility is to violate \textit{Universal Domain} and restrict attention to profiles with certain structure rather than allowing all logically possible profiles of complete and transitive preferences. If a voter's preferences fall into a restricted domain, the possibility of meaningful aggregation can be restored. The well-behaved restricted domains are called \textit{Condorcet domains} and examples include single-peaked, group-separable and value-restricted preferences \citep{PUPPE2024426}. Sen's theorem (\citeyear{sen}) provides a sufficient condition under which the pairwise majority relation yields a transitive ranking. While Sen's condition is very general, it is comparatively technical \citep{Elsholtz2005} and not interpretable in the benchmarking context. Thus, we focus on restricted domains with clearer meaning and investigate domains of single-peaked, group-separable and distance-restricted preferences.

\subsection{Aggregation Across Datasets}\label{section_aggregation_across_datasets}

So how can we go from a benchmark operator to a ranking of the models over all datasets? Prior to the aggregation across datasets, we impose the domain restrictions on the profile $R_D$ of rankings induced by metrics within each dataset. If $R_D$ is single-peaked, group-separable or distance-restricted to degree $1$, then the benchmark operator yields a coherent ranking. We then consider the family of orders generated by $B$, given by $\text{ord}(B,\mathcal{D}):= \bigl(B(R_D)\bigr)_{D \in \mathcal{D}},$ which contains one aggregate ranking $B(R_D)$ for each dataset. The domain restrictions in the previous step ensure that each $B(R_D)$ is a complete and transitive ranking. If we now want to choose \textit{one} ranking from this list, we should choose one that best represents the family. Under a statistical perspective, every dataset $D$ and its ranking $B(R_D)$ can be treated as a data point of some population (i.e., of $\mathcal{D}$) of interest. A natural approach is based on \textit{depth functions} \citep{zuo} and its generalizations to non-standard data spaces, such as relations. There are different depth functions available, for example the ufg depth \citep{BLOCHER2024109166,blocher2024union} or a generalization of Tukey depth (\citeyear{tukey1975mathematics}), firstly used in \citet{JANSEN201849} under the name \textit{commonality sharing rule} and analyzed in \citet{BLOCHER2025105372}, for other depth based approaches, cf., e.g., \cite{pmlr-v151-goibert22a,nehring2023multi}. In this paper, we use the generalized Tukey depth where the objects of interest are rankings. The idea is 
to choose from the population $\text{ord}(B,\mathcal{D})$ of rankings the so-called \textit{commonality sharing ranking(s)} -- the ranking(s) that share(s) with every subpopulation of minimum-size $k$ the commonalities of the subpopulation, where $k$ is chosen as small as possible. Here, commonalities are all pairs $(A,A^{'})$ which are ranked identically by all members of the subpopulation. This step requires the benchmark operator to produce exclusively acyclic relations: otherwise, the commonality sharing rule is not well defined.\footnote{The generality of the commonality sharing rule (and other methods like the ufg depth) would in principle allow one to extend the aggregation ideas to acyclic relations. However, the aggregated result would then generally be only a cyclic relation, which is not satisfying for a leaderboard ranking.}. Thus, the role of the domain restrictions is to ensure that the objects being aggregated across datasets are coherent rankings.

\section{Restricted Preference Domains in Benchmarking}\label{analysis}
\subsection{Single-Peaked Preferences}\label{singlepeaked}
Given fixed objects $\mathcal{A}$, $\mathcal{D}$ and $\Phi$, we now examine the structure in the set of preference profiles $\text{prof}(\mathcal{A},\mathcal{D},\Phi)$. We consider three types of domain restrictions: single-peakedness, group separability and distance-restrictedness. Conditional on suites with this specific structure, we show that  pairwise majority satisfies \textit{Non-Dictatorship}, \textit{IIA}, \textit{Weak Pareto} and \textit{Social Ordering}. Hence, this section turns the two properties in Section~\ref{motiv} into testable sufficient conditions.

The first and most famous restricted preference domain was established by \citet{black48} and independently discovered by \citet{arrow51}. Single-peakedness means that there exists a way to arrange models along some spectrum so that each metric's ranking of models has a single \textit{peak} on this spectrum. This can be viewed as the ``sweet spot" of each metric where the most preferred models for this metric lie, while performance decreases when moving further away from that region on the spectrum. In benchmark practice, this one-dimensional axis can be observed in several contexts, reflecting different trade-offs of model design. For instance, along an axis ordered by model size, interpretability and efficiency may prefer smaller models, while accuracy may peak at an intermediate or larger model size.

Our definition of single-peaked preference profiles follows \citet[Sec.~3]{ballester11}.  Assume $S \subseteq \mathcal{A}$ (with $|S| \geq 3$). Let $\mathcal{L}$ be the set of all strict linear orders over $S$ and set $L \in \mathcal{L}$. For any two models $A,A^{'} \in \mathcal{A}$, we write $A\succ_L A^{'}$ if $A$ precedes $A^{'}$ in the ordering $L$. 

\begin{definition}
For any set  $S \subseteq \mathcal{A}$ let $w(S,P_{D,i})$ denote the \textbf{least preferred model} in $S$ according to the strict preference relation $P_{D,i}$ and let $b(S,P_{D,i})$ be the \textbf{most preferred model} in $S$ according to $P_{D,i}$. 
Let $L$ be an \textbf{admissible orientation of $S$ with respect to $P_{D,i}$} if for any three models $A,A^{'},A^{''} \in S$ with $A^{'}=b(S,P_{D,i})$ being the most preferred model in $S$ according to metric $i$ such that $A^{'}\succ_LA \succ_LA^{''}$ or $A^{''}\succ_LA \succ_LA^{'}$, we have $\forall A, A^{''} \in S$,  $A\succ_{P_{D,i}} A^{''}$. Denote the set of all admissible orientations of $S\subseteq \mathcal{A}$ with respect to $P_{D,i}$ by $\mathcal{L}_S (P_{D,i})$ and set $\mathcal{L}_S(R_{D})= \bigcap_{i \in [n]} \mathcal{L}_S (P_{D,i})$. 
A profile $R_D$ is called \textbf{single-peaked} if there exists an ordering $L$ over the set of models $\mathcal{A}$ with $L \in \mathcal{L}_\mathcal{A}(R_{D})$ such that $\mathcal{L}_\mathcal{A}(R_{D})\neq \varnothing$.
\end{definition}

\begin{theorem}[\citealp{black48}]
Fix a dataset $D\in\mathcal D$ and a finite set of models $\mathcal{A}$. Let $R_D$ be a single-peaked profile induced by $D \in \mathcal{D}$. Assume $n$ is odd, with $n\geq 3$. Then the weak pairwise majority relation $\succeq_M^D$ satisfies Non-Dictatorship, Weak Pareto, IIA and yields a complete and transitive order on $\mathcal{A}$.\footnote{The proof of this and the following theorems can be found in Appendix \ref{app:proofs}.}
\end{theorem}

\paragraph{Example.}
On \textit{Human Aging}, consider three models GPT-4-0613 ($A$), GPT-4-1106-Preview ($B$) and GPT-3.5-Turbo-0613 ($C$). Their metric rankings are
\[
\textit{Accuracy}: B\succ A\succ C,\quad
\textit{Time}: C\succ A\succ B,\quad
\textit{Length}: A\succ C\succ B.
\]
Each metric prefers a different model, yet the profile is single-peaked along the axis $B\succ_L A\succ_L C$.  Pairwise majority yields $A\succ C\succ B$. Our diagnostic thus identifies a single-peaked profile, leading to a well-behaved ranking. A practitioner who needs to choose one model can then select GPT-4-0613, which beats each other model by two metrics to one (see Appendix~\ref{app:metrics_obj}).

\subsection{Group-Separable Preferences}
\citet{inada64, inada69} has introduced the class of group-separable domains which are sufficient for a transitive social ranking and fulfillment of all Arrow's axioms except for \textit{Universal Domain}, provided that the number of voters is odd. In a group-separable profile, for any subset of models, there exists at least one group of models that every metric regards as entirely above (or entirely below) the rest. This is very plausible in cases where the metrics measure the same latent construct, for instance, when using several accuracy metrics. Even if the metrics are not identical, they typically agree that some models are better than the others. But the condition can also hold in multi-objective comparisons, such as combining accuracy and efficiency metrics when the benchmark contains models that are consistently dominated on the chosen metric set, e.g. models that tend to be less accurate and less efficient than several others. Whenever such dominated models are present, the metrics agree on a \textit{separation} that places them entirely below the rest (and similarly, if a subset contains models that are both accurate and efficient, they can form a block that all metrics place above the rest). We define group-separability following \citet[Sec.~4]{ballester11}.

\begin{definition}
Assume $S \subseteq \mathcal{A}$ (with $|S| \geq 3$). Let a non-empty set $E \subset S$ be a \textbf{separation with respect to $P_{D,i}$} if either $\forall A \in E, \forall A^{'} \in S \setminus E: A \succ_{P_{D,i}} A^{'}$ or $\forall A \in E, \forall A^{'} \in S \setminus E: A^{'} \succ_{P_{D,i}} A$. Let $\mathcal{S}_{P_{D,i}}$ be the set of all separations of $S$ with respect to $P_{D,i}$ and $\mathcal{S}_{R_{D}} = \bigcap_{i \in [n]} \mathcal{S}_{P_{D,i}}$. A profile $R_D$ is called \textbf{group-separable} if for any $S \subseteq \mathcal{A}$, $\mathcal{S}_{R_{D}} \neq \varnothing$. 
\end{definition}

\begin{theorem}[\citealp{inada64,inada69}]
Fix a dataset $D\in\mathcal D$ and a finite set of models $\mathcal{A}$.  Assume $n$ is odd, with $n\geq 3$. Let $R_D$ be a group-separable profile induced by $D \in \mathcal{D}$. Then the strict pairwise majority relation $\succ_M^D$ satisfies Non-Dictatorship, Weak Pareto, IIA and yields a complete and transitive order on $\mathcal{A}$.
\end{theorem}

\subsection{Distance-Restricted Preferences}
So far we have investigated two different origins of structure in benchmarks: structure from an axis over models resulting in each metric's ``sweet spot" and structure from clustering of models with respect to their performance on different metrics. Now we deal with structure emerging from a situation of \textit{near-consensus} among metrics. This is particularly valuable when benchmarking aims at comparing models in terms of a single evaluation criterion such as robust accuracy, which is too complex to be expressed in terms of a single metric.

\begin{definition}
A profile $R_D$ is called \textbf{distance-restricted to degree $p \in \mathbb{N}$}, if for all $i,j \in [n]$ we have that $d_S(P_{D,i},P_{D,j})\leq p$, where $d_S(P_{D,i},P_{D,j})$ is defined as $|\{ (A,A^{'}) \mid A \succ_{P_{D,i}} A^{'} \text{ and } A^{'} \succ_{P_{D,j}} A \} |,$ where $d_S(P_{D,i},P_{D,j})$ is the \textbf{swap distance} counting the number of disagreeing pairs between $P_{D,i}$ and $P_{D,j}$.
\end{definition}

The distance $d_S$ coincides with the Kendall $\tau$  distance between strict rankings \citep{kendall}. In social choice literature, it is also known as the Kemeny distance between two preference orders, measuring the number of pairwise comparisons on which two orders disagree \citep{kemeny59,kemeny1962mathematical,baigent87,bossert92, burak18}. 

A $p$-distance-restricted profile means that any two metrics disagree on at most $p$ pairwise comparisons of models. This is plausible when metrics assess the same latent construct. For instance, \textit{exact\_match}, \textit{prefix\_exact\_match} and \textit{quasi\_exact\_match} from HELM all measure whether the correct answer matches the model's prediction, differing only in how much extra text or small formatting variation they allow. Hence, we expect a near-consensus: the rankings differ only on a few pairwise disagreements. In the case of $p=1$, Condorcet cycles cannot appear and we get the following result.

\begin{theorem}
Fix a dataset $D\in\mathcal D$ and a finite set of models $\mathcal{A}$. Suppose $n\geq3$. Assume the induced profile $R_D$ is distance-restricted to degree $1$. Then, the weak pairwise majority relation $\succeq_M^D$ is a transitive and complete relation on $\mathcal A$. It satisfies Non-Dictatorship, Weak Pareto and IIA.
\end{theorem}

\section{Experiments}\label{experiments}
\subsection{Experimental Setup}\label{setupexp}
We use the MMLU benchmark under HELM (v1.0.0) which includes 23 models evaluated on 57 subjects across multiple evaluation metrics. For each run directory, we identify the MMLU subject as well as the evaluated model and extract the subject-specific mean metric values reported by HELM. When multiple runs exist for the same model-subject pair under different evaluation settings, we keep the most frequent setting and average across repeats. 

We consider 83 model sets: 23 sets are selected to represent common benchmark comparisons, such as comparing strong systems, models from the same provider and mixtures of open and closed models; the remaining 60 sets are random, sampled without replacement from the 22 models with full MMLU coverage, with 30 sets of size 3, 20 of size 4 and 10 of size 5 (see Appendix \ref{app:models_metrics} for the full list of models; the model \textit{google/gemini-pro} was excluded as HELM reports only 5 of 57 MMLU subjects for it). We use all $42$ metric subsets of sizes $3$--$6$ drawn from the set of six metrics, which includes four accuracy metrics \textit{exact\_match}, \textit{prefix\_exact\_match}, \textit{quasi\_exact\_match}, \textit{quasi\_prefix\_exact\_match} and two efficiency metrics \textit{inference\_runtime} and \textit{num\_bytes}. Metrics are voters, while models are alternatives. For each subject and each metric we sort the fixed models by their scores to obtain a preference ranking. We consider both the alphabetical and the reverse alphabetical tie-breaking rule for resolving equal scores. We test the three domain restrictions across all combinations of model sets, metric subsets and tie-breaking rules, resulting in $83 \times 42 \times 2 = 6972$ settings. We also repeat the analysis ten times with random tie-breaking and compare with matched uniformly random profiles.

For each configuration $(\mathcal{A}, \Phi, \text{tie})$, we compute the fraction of subjects with complete data on which the induced preference profile is single-peaked, group-separable and distance-restricted to degree 1. We also compute the mean Kendall $\tau$ between the metrics' rankings on a subject, averaged across the 57 subjects. We test for single-peakedness by enumerating all possible permutations of the models that correspond to potential axes and checking whether each metric’s ranking is single-peaked with respect to that axis. Using the strict metrics' rankings obtained after tie-breaking, we test for group separability. Within any set $S$ of models, we repeatedly search for a nontrivial subset $E \subset S$ that all metrics rank entirely above or entirely below $S\setminus E$. If such a separation exists, we split $S$ into $E$ and $S\setminus E$ and continue on each part until only sets of fewer than three models remain. To test for distance-restrictedness, we compute the swap distance $d_S$ between every pair of metrics' rankings for each task. A subject is then declared distance-restricted to degree 1 if the maximum swap distance is at most 1. Further experiments on suites such as PMLB and OpenML are reported in Appendix \ref{app:exp_other}. For some domains which satisfy our structural assumptions, we also compute the overall ranking across all datasets by using the commonality sharing aggregation rule described in Section~\ref{section_aggregation_across_datasets}. 

For each setting, the number of subjects (out of those with complete data for that configuration) on which a property holds is a binomial count and we report Wilson 95\% confidence intervals \citep{wilson, brown} on the corresponding probability -- the chance that the property holds on a single MMLU subject under that configuration. The Wilson interval reflects uncertainty at the level of each configuration: treating the 57 MMLU subjects as Bernoulli trials drawn from a hypothetical population of benchmarks like MMLU, it bounds the probability at which the property holds in that configuration. To measure whether the existence of structure in the preference domains is associated with agreement between metrics' rankings, we fit binomial mixed-effects regressions with a logit link for metric subsets of sizes $3$ and $4$. For each configuration, the response is the number of MMLU subjects on which the property holds out of the number of subjects with complete data. The predictors are the mean Kendall $\tau$, a dummy variable for whether the metric subset includes an efficiency metric, the model set size and the metric subset size. We add random intercepts for the model set and the metric subset, as the same models and metrics appear in multiple configurations, so they cannot be treated as independent observations.

\subsection{Results}\label{res}

We first report the results under alphabetical tie-breaking, thus referring to the $83 \times 42 = 3486$ combinations of model and metric sets. All configurations have complete data for all $57$ MMLU subjects. Each structural restriction occurs more often than under the matched uniformly random baseline at the level of profiles on individual subjects. Pairwise majority is transitive in $93.65\%$ of the $198702$ profiles on individual subjects, as compared with $66.64\%$ under the matched uniformly random baseline (Appendix~\ref{app:transitivity}). Single-peakedness holds on all $57$ subjects in $1426$ configurations ($40.9\%$), group separability in $1208$ configurations ($34.7\%$) and distance-restrictedness to degree $1$ in $461$ configurations ($13.2\%$). For many choices of models and metrics, these structural assumptions hold for every MMLU subject. For any such configuration in which the property holds on all 57 subjects, the Wilson 95\% confidence interval is $[0.937, 1.000]$, which means that observing the property on all MMLU subjects is compatible with an underlying probability as low as $0.937$. The strongest pattern appears for metric sets containing exclusively accuracy metrics. Among $415$ such combinations, single-peakedness holds throughout the suite in $411$ cases ($99.0\%$), group separability in $365$ cases ($88.0\%$) and distance-restrictedness in $329$ cases ($79.3\%$).  This is consistent with the interpretation that the accuracy metrics measure the same latent construct and therefore induce highly aligned rankings of models. Among the $3071$ configurations containing efficiency metrics such as \textit{Time} or \textit{Length}, single-peakedness still holds on all subjects in $1015$ cases ($33.1\%$) and group separability in $843$ cases ($27.5\%$). Distance-restrictedness holds in only $132$ cases ($4.3\%$). For $461$ out of $3486$ evaluation settings ($13.2\%$), all three domain restrictions are satisfied on all subjects; while for $1620$ ($46.5\%$) configurations, at least one of the three restrictions holds for all subjects.

\begin{table}[h]
\centering
\caption{Percentage of profiles on individual subjects satisfying each property, with the number of metrics fixed at three, across the 83 model sets and 57 subjects. Results use alphabetical tie-breaking (see Appendix~\ref{app:metrics_obj}).}
\label{tab:metric_groups}
\vspace{4pt}
\small
\begin{tabular}{@{}p{0.40\linewidth}rrr@{}}
\toprule
\textbf{Metric set} &
\shortstack[r]{\textbf{Single-}\\\textbf{peakedness}} &
\shortstack[r]{\textbf{Group}\\\textbf{separability}} &
\shortstack[r]{\textbf{Transitive}\\\textbf{pairwise majority}} \\
\midrule
Three accuracy metrics
& 99.98 & 99.77 & 100.00 \\
Two accuracy metrics, one efficiency metric
& 82.39 & 88.33 & 99.94 \\
One accuracy metric, two efficiency metrics
& 39.79 & 53.61 & 80.84 \\
\bottomrule
\end{tabular}
\end{table}

Table~\ref{tab:metric_groups} illustrates that the choice of metrics matters substantially for determining whether the overall ranking is well-behaved. Combining one accuracy metric with efficiency metrics such as \textit{Time} and \textit{Length} results in trade-offs, yet pairwise majority still induces a consistent ranking in $80.84\%$ of the profiles on individual subjects (see Appendix~\ref{app:metrics_obj} for cases in which each of the three metrics prefers a different model). The choice of both models and metrics is therefore an inherent part of the benchmark design. The mixed-effects regressions for metric subsets of sizes $3$ and $4$ support this interpretation. Higher mean Kendall $\tau$ between metric rankings is positively associated with all three domain restrictions, while the inclusion of efficiency metrics and larger model sets are negatively associated with them. For single-peakedness and distance-restrictedness, larger metric sets are also negatively associated. All of these associations are statistically significant in the mixed-effects models and the signs of the coefficients remain unchanged in the quasi-binomial regression conducted for sensitivity analysis to allow for overdispersion (see Appendix~\ref{app:regressions} for all regression results). Using \textit{Accuracy}, \textit{Time} and \textit{Length}, we also examine sensitivity to changes in \textit{Time} and to paired bootstrap samples of MMLU questions (see Appendix~\ref{app:sensitivity}). We find that when all three metrics are recalculated on paired bootstrap samples of questions, the original winner model beats every other model in $87.1\%$ of samples on average, while the entire ranking remains unchanged in $75.6\%$, averaged across the $3825$ profiles where the original majority ranking is transitive. Benchmark designers can thus examine both whether the criteria induce a consistent ranking and how sensitive the selected model and the complete ranking are to changes in the evaluation questions.

\section{Discussion}\label{disc}
What conditions must multi-criteria benchmarks satisfy to produce well-behaved rankings? When benchmark designers want one ranking across competing objectives, the key question is whether the rankings have enough structure to make a single ordering defensible. Rank-based rules such as Borda score and mean win rate evaluate a model through its position relative to the whole comparison set; consequently, the ordering of two fixed models can flip when a third model enters. Pairwise majority avoids this, and its transitivity can be justified by structure: without a restricted domain, those pairwise wins can form cycles. Moreover, a complete and transitive strict majority ranking minimizes the total number of pairwise disagreements with the metric rankings \citep{young1978consistent}, and it can be interpreted as an overall agreement with the criteria. Our design recommendation is thus a diagnostic protocol: convert each metric on each task into a ranking over the candidate models and test whether the resulting profile satisfies one of the discussed domain restrictions. If this diagnostic passes across all instances, pairwise majority gives coherent dataset rankings that can be combined across datasets by commonality sharing. Since the conditions are sufficient but not necessary, their failure does not rule out the possibility of a well-behaved ranking.

\begin{enumerate}
\item If none of the conditions is satisfied, test the pairwise majority relation directly. If it is complete and transitive, report the resulting ranking.

\item If the full ranking is not transitive, check whether a model nevertheless beats every other model.

\item If no such model exists, take the smallest group of models in which every model beats every model outside the group together with the conflicting comparisons within it. This yields a shortlist of models when the group is smaller than the original model set. Further selection can then depend on the application's priorities. If the main objective of the evaluation is accuracy, then the model with the highest accuracy in this group should be selected.

\end{enumerate}

The third step is based on the concept of a Smith set \citep[Sec.~3.3.3]{brandt2016tournament}. Finally, one should also test how sensitive the chosen model or ranking is to variations in the evaluation questions or metric values (see Appendix~\ref{app:sensitivity}). Our tests on HELM MMLU, PMLB and OpenML illustrate these possibilities (see Appendix~\ref{app:selection_possibilities}). On OpenML, although only four of the 80 datasets satisfy a sufficient condition, 40 induce a well-behaved ranking. On another 21 datasets, a model beats every other model despite the full ranking not being transitive. Of the remaining 19 datasets, 18 admit a smaller shortlist under the third step of the procedure. Several limitations qualify our conclusions. First, our results are conditional: the restricted domains are sufficient conditions for well-behaved pairwise majority aggregation, but they are not necessary. Second, our empirical analysis is limited mainly to HELM MMLU, while the PMLB and OpenML results suggest that such structure is not equally prevalent across domains.

\section*{Reproducibility Statement}
Proofs are provided in Appendix~\ref{app:proofs}. Section~\ref{setupexp} and Appendix~\ref{app:exphelm} describe preprocessing, model and metric sets and tie-breaking. Appendices~\ref{app:exp_other}, ~\ref{app:selection} and~\ref{app:expmotiv} describe the additional experiments. Code is available in the \href{https://github.com/polinamgordienko-glitch/Beyond-Arrow-From-Impossibility-to-Possibility-in-Multi-Criteria-Benchmarking.gitF}{code repository}.

\bibliographystyle{iclr2027_conference}
\bibliography{bibliography}

\newpage
\appendix
\section{Proofs}\label{app:proofs}
\subsection{Single-Peaked Preferences}\label{appsp}
 Fix a dataset $D$, a set of models $\mathcal{A}$, and a multivariate performance measure $\Phi=(\phi_1 , \dots , \phi_n)$. After applying a fixed deterministic tie-breaking rule, each metric $i$ induces a strict linear order $P_i$ on $\mathcal{A}$.
\begin{lemma} \cite{black48}
\label{lemsp1}
Let $R_D$ be a single-peaked profile. Assume $n$ is odd (with $n \geq 3)$. Then the pairwise majority relation $\succeq_M^D$ is transitive.
\end{lemma}
\begin{proof}
Since $R_D$ is single-peaked, there exists an ordering $L$ over $\mathcal{A}$ with $L \in \mathcal{L}_\mathcal{A}(R_{D})$. By \citet{black48} the pairwise majority is transitive for profiles that are single-peaked with respect to a common axis. Hence, $\succeq_M^D$ is transitive.
\end{proof}
\begin{theorem}
Let $R_D$ be a single-peaked profile induced by dataset $D \in \mathcal{D}$. Assume $n$ is odd (with $n \geq 3)$.  Then the pairwise majority relation $\succeq_M^D$ satisfies Non-Dictatorship, Weak Pareto, IIA and yields a complete and transitive relation on $\mathcal{A}$.
\end{theorem}
\begin{proof}
The pairwise majority relation $\succeq_M^D$ is complete by definition. By Lemma~\ref{lemsp1} it is transitive. Furthermore, the relation $\succeq_M^D$ satisfies \textit{IIA} by construction. \textit{Weak Pareto} holds because if every metric prefers model $A$ to model $A^{'}$, then the pairwise majority count for $A$ over $A^{'}$ is $n$, and $A\succ_M^DA^{'}$. For \textit{Non-Dictatorship} consider the following. Fix a metric $i$. Since $R_D$ is single-peaked, there exists an ordering $L$ over $\mathcal{A}$ with $L \in \mathcal{L}_\mathcal{A}(R_{D})$. Assume that three models $A,A^{'},A^{''}$ are ordered on $L$ such that $A\succ_L A^{'} \succ_L A^{''}$. Further assume that a metric $i$ has its peak at $A$ so that $A \succ_{P_{D,i}} A^{'} \succ_{P_{D,i}} A^{''}$ and every other metric $j \neq i$ has its peak  at $A^{''}$ so that $A^{''} \succ_{P_{D,j}} A^{'} \succ_{P_{D,j}} A$. Then, pairwise majority over all metrics does not rank $A$ over $A^{''}$. It follows that metric $i$ is not a dictator. Since $i$ was arbitrarily chosen, no dictator exists.
\end{proof}

\subsection{Group-Separable Preferences}\label{appgs}
 Fix a dataset $D$, a set of models $\mathcal{A}$, and a multivariate performance measure $\Phi=(\phi_1 , \dots , \phi_n)$. After applying a fixed deterministic tie-breaking rule, each metric $i$ induces a strict linear order $P_i$ on $\mathcal{A}$.
\begin{lemma}\cite{inada64,inada69}\label{lemgs1}
Let $R_D$ be a group-separable profile. Assume $n$ is odd (with $n \geq 3)$. Then the strict pairwise majority relation $\succ_M^D$ is transitive.
\end{lemma}
\begin{proof}
By \citet{ballester11}, the profile $R_D$ is medium-restricted, since this condition is implied by group separability. By \citet{inada64,inada69}, the pairwise majority relation is transitive for a medium-restricted profile, given an odd number of metrics under consideration. Hence, $\succ_M^D$ is transitive.
\end{proof}
\begin{theorem}\label{gs}
Let $R_D$ be a group-separable profile induced by dataset $D \in \mathcal{D}$. Assume $n$ is odd (with $n \geq 3$). Then the pairwise majority relation $\succ_M^D$ satisfies Non-Dictatorship, Weak Pareto, IIA and yields a complete and transitive order on $\mathcal{A}$.
\end{theorem}
\begin{proof}
Fix $A \neq A^{'}$. Since each metric induces a strict linear order $P_i$, for every $i \in [n]$, exactly one of $A\succ_{P_i}A^{'}$ and $A^{'}\succ_{P_i}A$ holds. Hence, the two majority counts sum up to $n$. Since $n$ is odd, either $A\succ_M^D A^{'}$ or $A^{'} \succ_M^D A$ holds. It follows that the pairwise majority relation $\succ_M^D$ is complete. By Lemma~\ref{lemgs1} it is transitive. The relation $\succ_M^D$ fulfills the condition \textit{IIA} by construction. The condition \textit{Weak Pareto} is satisfied because if every metric prefers model $A$ to model $A^{'}$, then the pairwise majority count for $A$ over $A^{'}$ is exactly $n$, and, thus, $A\succ_M^DA^{'}$. For \textit{Non-Dictatorship} consider the following. Fix a metric $i$. Consider three models $A,A^{'},A^{''} \in \mathcal{A}$. Assume that metric $i$ ranks $A \succ_{P_{D,i}} A^{'} \succ_{P_{D,i}} A^{''}$ and every other metric $j \neq i$ ranks $A^{'} \succ_{P_{D,j}} A \succ_{P_{D,j}} A^{''}$. Since $A^{''}$ is ranked consistently at the bottom by every metric, $\{A^{''}\}$ is a separation of $\{A,A^{'},A^{''}\}$ with respect to $P_{D,l}$ for all $l \in [n]$. Consider all other models $Z \in \mathcal{A} \setminus \{A,A^{'},A^{''}\}$. Fix $Z= \{Z_1, \dots Z_q \}$ and assume $Z_1 \succ_{P_{D,l}} \dots \succ_{P_{D,l}} Z_q$ for all $l \in [n]$, with all $Z$ ranked above $A, A^{'}, A^{''}$ for every $l$. Thus, the resulting profile is group-separable by definition. The pairwise majority over all metrics ranks $A^{'}$ over $A$. Therefore, metric $i$ is not a dictator. Since $i$ was arbitrarily chosen, no dictator exists.
\end{proof}

\subsection{Distance-Restricted Preferences}\label{appds}

Fix a dataset $D$, a set of models $\mathcal{A}$, and a multivariate performance measure $\Phi=(\phi_1 , \dots , \phi_n)$. Assume $n\geq 3$. After applying a fixed deterministic tie-breaking rule, each metric $i$ induces a strict linear order $P_i$ over the finite set of models $\mathcal{A}$. Let $P_i,P_j$ be such strict linear orders on $\mathcal{A}$ with $P_i := P_{D,i}$ for all $i$. Let $d_S(P_i,P_j)$ be the number of pairwise comparisons between models in $\mathcal{A}$ on which $P_i$ and $P_j$ disagree. Assume the induced profile $R_D$ is $1$-distance-restricted so that $d_S(P_i,P_j)\leq1$ for all $i,j$.
\begin{lemma}
\label{lemd1}
If $d_S(P_i,P_j)\leq1$, then either $P_i=P_j$ or $P_j$ is obtained from $P_i$ by swapping exactly two adjacent models in $P_i$. 
\end{lemma}
\begin{proof}
Assume $d_S(P_i,P_j)\leq1$ and $P_i\neq P_j$. Then there exists a pair of models $A,A^{'} \in \mathcal{A}$ such that $P_i$ ranks $A$ above $A^{'}$ and $P_j$ ranks $A^{'}$ above $A$. If models $A$ and $A^{'}$ are not adjacent in the order $P_i$, then there must exist at least one model $A^{''}$ that lies between $A$ and $A^{'}$ in $P_i$. For $P_j$ to rank $A^{'}$ above $A$, $A^{'}$ must move past $A^{''}$ or $A$ must move past $A^{''}$. This, however, changes at least one additional pairwise comparison relative to $P_i$. We get two pairs on which $P_i$ and $P_j$ disagree which contradicts $d_S(P_i,P_j)\leq1$. Therefore, $A$ and $A^{'}$ are adjacent in $P_i$. Since there is only one disagreement, $P_j$ is obtained from $P_i$ by swapping exactly that pair of models in $P_i$.
\end{proof}
\begin{lemma}
\label{lemd2}
Fix an order $P:=P_1$ such that $d_S(P_i,P_1)\leq1$ for all $i$. There exists a pair $A,A^{'} \in \mathcal{A}$ such that for all $i$ either $P_i=P_1$ or $P_i$ is obtained from $P_1$ by swapping the pair $A,A^{'}$.
\end{lemma}
\begin{proof}
By $d_S(P_i,P_1)\leq1$ and Lemma~\ref{lemd1}, we must have either $P_i=P_1$ or $P_i$ is obtained from $P_1$ by swapping exactly one adjacent pair of models in $P_1$. Assume for contradiction that there exist metrics $i \neq j$ such that $P_i$ is obtained from $P_1$ by swapping one adjacent pair $A,A^{'}$, and that $P_j$ is obtained from $P_1$ by swapping a different pair of models $Q,Q^{'}$. Then, $P_i$ and $P_j$ disagree on the comparison between $A$ and $A^{'}$ and, additionally, on the comparison between $Q$ and $Q^{'}$. Thus, there are two pairs of models on which $P_i$ and $P_j$ disagree, which contradicts the assumption $d_S(P_i,P_j)\leq1$ for all $i,j$. Therefore, all orders other than $P_1$ must swap the same adjacent pair of models $A,A^{'}$.
\end{proof}
\begin{lemma}
\label{lemd3}
The pairwise majority relation $\succeq_M^D$ is transitive.
\end{lemma}
\begin{proof}
By Lemma~\ref{lemd2}, for all $i$, $P_i$ coincides with $P_1$ on every pairwise comparison except possibly the single pair of models $A,A^{'}$. It follows that for any pair of models $Z,Z^{'} \in \mathcal{A}$ different from $A,A^{'}$, every metric agrees on the ranking of $Z$ compared to $Z^{'}$. Thus, majority agrees on the ranking of $Z$ compared to $Z^{'}$ as well and ranks both models exactly the same way as in the order $P_1$. Therefore, the only pair on which the majority outcome might differ from $P_1$ is $A,A^{'}$. The majority can either rank $A$ above $A^{'}$ or $A^{'}$ above $A$ or tie them. In all cases, for every third model $Z$, the comparisons between $Z$ and $A$ and between $Z$ and $A^{'}$ are exactly as in $P_1$. Since at most one pairwise comparison can differ from $P_1$, a cycle involving three distinct models cannot exist. Hence, the pairwise majority relation $\succeq_M^D$ is transitive.
\end{proof}
\begin{theorem}
Fix a dataset $D\in\mathcal D$ and a finite set of models $\mathcal{A}$. Suppose $n\geq3$. Assume the induced profile $R_D$ is distance-restricted to degree $1$. Then, the pairwise majority relation $\succeq_M^D$ is a transitive and complete relation on $\mathcal A$. It satisfies Non-Dictatorship, Weak Pareto and IIA.
\end{theorem}
\begin{proof}
The relation $\succeq_M^D$ is complete by definition. By Lemma~\ref{lemd3}, $\succeq_M^D$ is transitive. The relation $\succeq_M^D$ fulfills \textit{IIA} by construction. \textit{Weak Pareto} holds because if all metrics strictly prefer model $A$ to model $A^{'}$, then the pairwise majority count for $A$ over $A^{'}$ is exactly $n$ and $A\succ_M^DA^{'}$. For \textit{Non-Dictatorship} consider the following. Fix a metric $i$. Consider the fixed order $P_1$. Let $A,A^{'}$ be the pair of models from Lemma~\ref{lemd2} such that the only possible deviation from $P_1$ is swapping $A$ and $A^{'}$. Consider a preference profile where metric $i$ ranks $A$ over $A^{'}$ and every other metric $j\neq i$ ranks $A^{'}$ over $A$. Since the orders differ on at most one swapped pair, the profile is distance-restricted to degree $1$. The pairwise majority over all metrics does not rank $A$ over $A^{'}$: metric $i$ ranks $A$ over $A^{'}$ but the collective preference over all metrics is $A^{'}$ over $A$. Hence, metric $i$ is not a dictator. Since $i$ was arbitrarily chosen, there is no dictator.
\end{proof}

\section{Experiments on Restricted Preference Domains: HELM MMLU}\label{app:exphelm}
\subsection{Models and Metrics}\label{app:models_metrics}
The curated model sets $\mathcal{A}_1,\dots,\mathcal{A}_{23}$ used in Section~\ref{experiments} are as follows. Several of these sets are intentionally nested. This lets us complement the random model sets with controlled perturbations of the alternative set, so that the experiments can probe how the observed domain structure changes by adding a small number of models.

Model set $\mathcal{A}_1 = \{$GPT-4-0613, GPT-4-1106-Preview, Claude-3-Opus-20240229, Claude-3-Sonnet-20240229, Qwen1.5-72b, Llama-2-70b, Claude-3-Haiku-20240307, Yi-34b, Yi-6b$\}$.

Model set $\mathcal{A}_2 = \{$GPT-4-0613, GPT-4-1106-Preview, Claude-3-Opus-20240229, Claude-3-Sonnet-20240229, Qwen1.5-72b, Llama-2-70b, Mistral-7b-v0.1$\}$.

Model set $\mathcal{A}_3 = \{$GPT-4-0613, Qwen1.5-72b, Llama-2-70b, Mistral-7b-v0.1$\}$.

Model set $\mathcal{A}_4 = \{$GPT-4-0613, GPT-4-1106-Preview, Claude-3-Opus-20240229, Claude-3-Sonnet-20240229, Qwen1.5-72b$\}$.

Model set $\mathcal{A}_5 = \{$Claude-3-Opus-20240229, GPT-4-0613, GPT-4-1106-Preview, Google-Text-Unicorn@001$\}$.

Model set $\mathcal{A}_6 = \{$Claude-3-Opus-20240229, GPT-4-0613, GPT-4-1106-Preview, Google-Text-Unicorn@001, Qwen1.5-72b$\}$.

Model set $\mathcal{A}_7 = \{$Claude-3-Opus-20240229, GPT-4-0613, GPT-4-1106-Preview, Google-Text-Unicorn@001, Qwen1.5-72b, Yi-34b$\}$.

Model set $\mathcal{A}_8 = \{$Claude-3-Opus-20240229, GPT-4-0613, GPT-4-1106-Preview, Google-Text-Unicorn@001, Qwen1.5-72b, Yi-34b, Claude-3-Sonnet-20240229$\}$.

Model set $\mathcal{A}_9 = \{$Claude-3-Opus-20240229, Mistral-7b-v0.1, GPT-4-0613, Llama-2-13b$\}$.

Model set $\mathcal{A}_{10} = \{$Claude-3-Opus-20240229, Mistral-7b-v0.1, GPT-4-0613, Llama-2-13b, GPT-4-1106-Preview$\}$.

Model set $\mathcal{A}_{11} = \{$Claude-3-Sonnet-20240229, Mistral-7b-v0.1, GPT-4-0613, Llama-2-70b, GPT-4-1106-Preview$\}$.

Model set $\mathcal{A}_{12} = \{$GPT-4-0613, GPT-4-1106-Preview, Claude-3-Opus-20240229, Qwen1.5-72b$\}$.

Model set $\mathcal{A}_{13} = \{$GPT-4-0613, Qwen1.5-72b, Yi-6b$\}$.

Model set $\mathcal{A}_{14} = \{$Qwen1.5-72b, Llama-2-70b, Mistral-7b-v0.1, Yi-34b, Yi-6b$\}$.

Model set $\mathcal{A}_{15} = \{$GPT-4-0613, GPT-4-1106-Preview, Claude-3-Opus-20240229, Claude-3-Sonnet-20240229, Claude-3-Haiku-20240307$\}$.

Model set $\mathcal{A}_{16} = \{$GPT-4-0613, GPT-4-1106-Preview, Claude-3-Opus-20240229, Claude-3-Sonnet-20240229, Qwen1.5-72b, Llama-2-70b$\}$.

Model set $\mathcal{A}_{17} = \{$Claude-3-Opus-20240229, Llama-2-70b, Mistral-7b-v0.1, Yi-34b, Yi-6b$\}$.

Model set $\mathcal{A}_{18} = \{$Claude-3-Haiku-20240307, Mistral-7b-v0.1, Llama-2-13b, Yi-6b$\}$.

Model set $\mathcal{A}_{19} = \{$Claude-3-Haiku-20240307, Claude-3-Sonnet-20240229, Claude-3-Opus-20240229$\}$.

Model set $\mathcal{A}_{20} = \{$Qwen1.5-14b, Qwen1.5-72b, Qwen1.5-7b$\}$.

Model set $\mathcal{A}_{21} = \{$GPT-4-0613, GPT-4-1106-Preview, GPT-3.5-Turbo-0613$\}$.

Model set $\mathcal{A}_{22} = \{$Llama-2-13b, Llama-2-70b, Llama-2-7b$\}$.

Model set $\mathcal{A}_{23} = \{$Mistral-7b-v0.1, Llama-2-13b, Yi-34b, Llama-2-70b, Qwen1.5-72b$\}$.

The remaining $60$ model sets are random sets generated in the code after setting \texttt{set.seed(13)}. They are denoted by their code identifiers $\texttt{R\_3\_01},\dots,\texttt{R\_3\_30}$, $\texttt{R\_4\_01},\dots,\texttt{R\_4\_20}$ and $\texttt{R\_5\_01},\dots,\texttt{R\_5\_10}$, corresponding to $30$ random sets of size $3$, $20$ random sets of size $4$ and $10$ random sets of size $5$. Each random set is sampled without replacement from the $22$ models with complete MMLU coverage after excluding \textit{google/gemini-pro}.

The main HELM MMLU analysis uses six metrics: four accuracy metrics $\{$\textit{exact\_match}, \textit{prefix\_exact\_match}, \textit{quasi\_exact\_match}, \textit{quasi\_prefix\_exact\_match}$\}$ and two efficiency metrics $\{$\textit{inference\_runtime}, \textit{num\_bytes}$\}$. We refer to \textit{inference\_runtime} and \textit{num\_bytes} as \textit{Time} and \textit{Length} respectively in the main text. Since lower values are preferable for the efficiency metrics, their signs are flipped before rankings are constructed. The metric sets are all subsets of these six metrics of size $3$, $4$, $5$ and $6$. Thus, we have $\binom{6}{3}=20$ metric triples, $\binom{6}{4}=15$ metric quadruples, $\binom{6}{5}=6$ metric subsets of size $5$ and one subset containing all six metrics, for $42$ metric sets in total. Of these, $\binom{4}{3}+\binom{4}{4}=5$ contain only accuracy metrics and the remaining $37$ include at least one efficiency metric. Crossing these metric sets with the $83$ model sets gives $3486$ combinations and $198,702$ profiles on individual subjects for each tie-breaking rule.

\subsection{Matched Uniformly Random Profiles}\label{app:uniformly_random}
For each combination of model set size and metric set size, we consider profiles in which each metric independently induces a uniformly random strict ranking over the models. We compute the probabilities exactly where possible and use $1{,}000{,}000$ uniformly random profiles for each remaining combination of sizes. The aggregate baseline uses the same distribution of model set and metric set sizes as the HELM MMLU analysis. The comparisons in Table~\ref{tab:sp_random_baseline} concern profiles on individual subjects and do not assume independence across MMLU subjects.

\begin{table}[h]
\caption{Single-peakedness in HELM MMLU and matched uniformly random profiles. Entries are percentages of profiles on individual subjects. HELM MMLU uses alphabetical tie-breaking.}
\label{tab:sp_random_baseline}
\centering
\begin{tabular}{@{}lrr@{}}
\toprule
Experimental setting & HELM MMLU & Uniformly random \\
\midrule
3 models, 3 metrics & 92.566 & 77.778 \\
4 models, 3 metrics & 78.263 & 36.111 \\
5 models, 3 metrics & 58.821 & 11.028 \\
6 models, 3 metrics & 39.386 & 2.369 \\
7 models, 3 metrics & 33.289 & 0.378 \\
9 models, 3 metrics & 22.632 & 0.0038 \\
\midrule
All model sets, 3 metrics & 77.387 & 46.133 \\
All model sets, 4 metrics & 66.271 & 27.840 \\
All model sets, 5 metrics & 53.629 & 17.607 \\
All model sets, 6 metrics & 39.400 & 11.425 \\
\bottomrule
\end{tabular}
\end{table}

With three models and three metrics, a uniformly random profile is single-peaked with probability $168/216=77.778\%$. The corresponding HELM MMLU fraction is $36{,}934/39{,}900=92.566\%$. This comparison differs from the criterion that a restriction holds on all $57$ subjects: among the $700$ configurations with three models and three metrics, single-peakedness holds on all subjects in $530/700=75.714\%$.

Under alphabetical tie-breaking, across metric sets of sizes $3$--$6$, the HELM MMLU fractions of profiles satisfying single-peakedness, group separability and distance-restrictedness to degree $1$ are $69.119\%$, $76.920\%$ and $28.084\%$, respectively. The corresponding matched uniformly random fractions are $34.698\%$, $36.641\%$ and $5.379\%$.

Using the criterion that single-peakedness holds on all $57$ subjects, the fractions of configurations are $52.651\%$, $35.904\%$, $20.281\%$ and $4.819\%$ for $3$, $4$, $5$ and $6$ metrics, respectively. Across all $42$ metric subsets, single-peakedness holds on all $57$ subjects in $1426/3486=40.906\%$ of configurations.

The prevalence decreases as the number of models or metrics increases, and some of the largest combinations have no observed single-peaked profiles. Every subset of size $5$ or $6$ contains an efficiency metric, so the number of metrics and the inclusion of efficiency metrics cannot be treated as independent in these aggregate results. 

\subsection{Results for Group Separability and Distance-Restrictedness}\label{app:gs_dr_res} Figure~\ref{group_sep} is a heatmap over combinations of model set and metric set for group separability. Figure~\ref{distance_restr} illustrates the results for distance-restrictedness to degree $1$. Group separability holds on all $57$ MMLU subjects in $1208$ out of $3486$ configurations, while distance-restrictedness to degree $1$ holds on all $57$ subjects in $461$ configurations. The low prevalence of distance-restrictedness in evaluation settings with accuracy and efficiency metrics is expected. This property requires metric rankings to be nearly identical. A highly accurate model can be slower or produce longer outputs, while a faster or more concise model can be efficient but less accurate. Such trade-offs change many pairwise comparisons at once, so the resulting rankings do not satisfy distance-restrictedness to degree $1$. Thus, the low prevalence of distance-restrictedness in such configurations is consistent with trade-offs between these evaluation dimensions. Furthermore, note that group separability itself is a structural property of a preference profile and can be tested for metric sets of any size. However, the guarantees for properties of pairwise majority in Theorem~\ref{gs} assume an odd number of metrics and group separability regressions restrict attention to the $1660$ configurations with three metrics (see Tables~\ref{tab:mixed_regressions} and ~\ref{tab:quasibinomial}).

\begin{figure}
    \centering
    \includegraphics[width=0.7\linewidth]{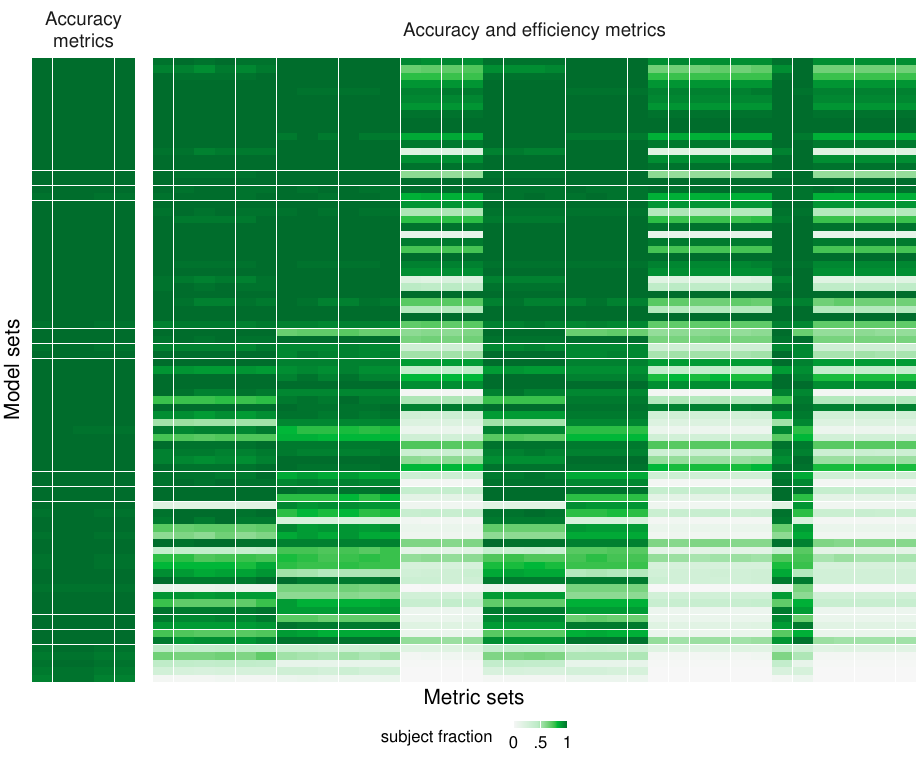}
    \caption{Group-separable structure in HELM MMLU. Each tile is one configuration of model and metric set under the alphabetical tie-breaker; the color gives the fraction of MMLU subjects whose induced preference profile is group-separable. Rows and columns use the same ordering as the single-peakedness heatmap in Figure~\ref{single-peaked}.}
    \label{group_sep}
\end{figure}

\begin{figure}
    \centering
    \includegraphics[width=0.7\linewidth]{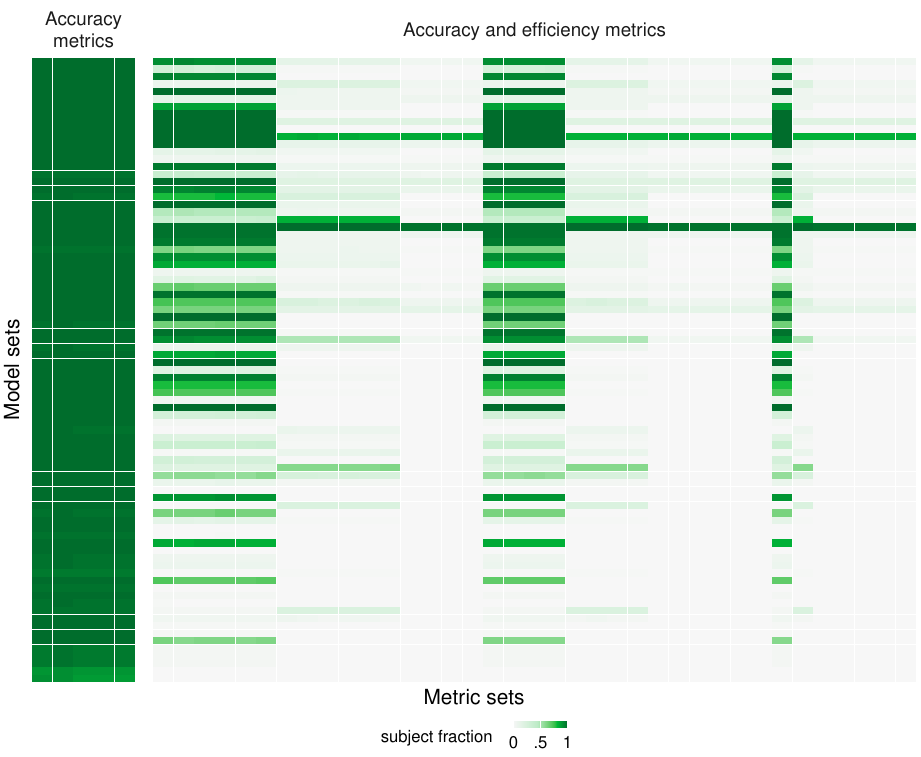}
    \caption{Distance-restrictedness to degree $1$ in HELM MMLU. Each tile is one configuration of model and metric set under the alphabetical tie-breaker; the color gives the fraction of MMLU subjects whose induced preference profile is distance-restricted to degree $1$. Rows and columns use the same ordering as the single-peakedness heatmap in Figure~\ref{single-peaked}.}
    \label{distance_restr}
\end{figure}

\subsection{Results under Different Tie-Breaking Rules}\label{app:tie_res}
Under reverse alphabetical tie-breaking, single-peakedness holds on all $57$ subjects in $1342$ configurations, group separability in $1205$ and distance-restrictedness to degree $1$ in $468$, compared to $1426$, $1208$ and $461$ under alphabetical tie-breaking. The respective numbers of configurations satisfying the property under both rules are $1208$, $967$ and $392$. Thus, the aggregate prevalence is similar across the deterministic rules, although some individual configurations change status.

For random tie-breaking, we draw independent random keys for each model, metric and subject within each configuration and use them only to resolve equal scores. We repeat the analysis ten times across all $83$ model sets and all $42$ metric subsets. The mean fractions of profiles satisfying single-peakedness, group separability and distance-restrictedness to degree $1$ are $70.011\%$, $74.925\%$ and $27.952\%$, respectively. The corresponding mean fractions of configurations satisfying the properties on all $57$ subjects are $30.341\%$, $24.656\%$ and $10.433\%$. The all-subject criterion is therefore more sensitive to tie-breaking than the aggregate fraction of profiles satisfying each restriction.

\subsection{Regressions}\label{app:regressions} As discussed in Section~\ref{experiments}, we fit binomial mixed-effects logistic regressions to study which features of a configuration are associated with the fraction of MMLU subjects satisfying each domain restriction (see Table~\ref{tab:mixed_regressions}). These regressions use alphabetical tie-breaking and metric subsets of sizes $3$ and $4$, giving $2905$ configurations. Within this analysis, the group-separability regressions use only the $1660$ configurations with three metrics. The predictors are mean Kendall $\tau$ between the metric rankings, a dummy for whether the metric set includes an efficiency metric, the number of models in the model set and the number of metrics in the metric subset. For group separability, the number of metrics is therefore constant and is omitted as a predictor. As a sensitivity analysis, we also fit quasi-binomial logistic regressions. For each combination of model set and metric subset under alphabetical tie-breaking, the response is a count: how many subjects satisfy the domain restriction out of the MMLU subjects with complete data. Hence, a binomial model is a natural regression choice. However, the MMLU subjects cannot necessarily be treated as independent and identical Bernoulli trials. Some subjects may be more similar to each other than others and some configurations of model and metric sets may have heterogeneity that is not fully captured by the predictors. We therefore fit quasi-binomial regressions to allow for overdispersion. We find that the coefficient signs are unchanged (see Table~\ref{tab:quasibinomial}), implying that our conclusions do not depend on the binomial model's variance assumption.

\begin{table}[h]
    \caption{Generalized linear mixed models with binomial response and logit link fitted to the HELM MMLU model and metric set combinations. The regressions use alphabetical tie-breaking and metric subsets of sizes $3$ and $4$; for group separability, only subsets of size $3$ are included. The table reports how metric agreement measured by Kendall $\tau$, the inclusion of efficiency metrics, model set size and metric set size are associated with each domain restriction. Coefficients are on the log-odds scale, with standard errors in parentheses. All reported fixed-effect coefficients are significant at $p<0.001$.}
    \label{tab:mixed_regressions}
    \begin{center}
    \begin{tabular}{@{}lccc@{}}
     \toprule
      & Single-peaked & Group-separable & Distance-restricted \\
      \midrule
    Mean Kendall $\tau$        & $6.649$ $(0.112)$  & $0.795$ $(0.112)$  & $9.556$ $(0.104)$ \\
Efficiency metric          & $-5.390$ $(0.671)$ & $-5.413$ $(0.676)$ & $-4.110$ $(0.183)$ \\
Number of models           & $-2.353$ $(0.182)$ & $-1.751$ $(0.162)$ & $-2.013$ $(0.110)$ \\
Number of metrics          & $-1.507$ $(0.303)$ & --  & $-1.451$ $(0.087)$ \\
Intercept                  & $20.176$ $(1.411)$ & $14.962$ $(0.911)$ & $10.956$ $(0.571)$ \\
\midrule
Configurations             & $2905$             & $1660$             & $2905$ \\
Model set groups           & $83$               & $83$               & $83$ \\
Metric set groups          & $35$               & $20$               & $35$ \\
SD, model set intercept    & $1.843$            & $1.629$            & $1.092$ \\
SD, metric set intercept   & $0.845$            & $1.153$            & $0.233$ \\
AIC                        & $15139.2$          & $8431.9$           & $9336.9$ \\
\bottomrule
\end{tabular}
\end{center}
\end{table}

\begin{table}[h]
    \caption{Quasi-binomial generalized linear models with logit link fitted to the HELM MMLU model and metric set combinations. The regressions use alphabetical tie-breaking and metric subsets of sizes $3$ and $4$; for group separability, only subsets of size $3$ are included. Coefficients are on the log-odds scale, with standard errors in parentheses. All reported coefficients are significant at $p<0.001$.}
    \label{tab:quasibinomial}
    \begin{center}
    \begin{tabular}{@{}lccc@{}}
     \toprule
 & Single-peaked & Group-separable & Distance-restricted \\
    \midrule
Mean Kendall $\tau$        & $5.870$ $(0.144)$  & $2.304$ $(0.139)$  & $9.352$ $(0.487)$ \\
Efficiency metric          & $-8.010$ $(1.899)$ & $-4.479$ $(0.620)$ & $-7.832$ $(1.215)$ \\
Number of models           & $-1.833$ $(0.042)$ & $-1.194$ $(0.043)$ & $-2.250$ $(0.162)$ \\
Number of metrics          & $-1.402$ $(0.064)$ & --                 & $-1.279$ $(0.193)$ \\
Intercept                  & $19.726$ $(1.936)$ & $10.421$ $(0.677)$ & $15.098$ $(1.745)$ \\
\midrule
Configurations             & $2905$             & $1660$             & $2905$ \\
Dispersion                 & $14.055$           & $14.523$           & $74.670$ \\
    \bottomrule
    \end{tabular}
     \end{center}
\end{table}

\subsection{Sensitivity to Changes in Runtime and Evaluation Questions}\label{app:sensitivity}

We examine all 83 model sets on all 57 subjects under the three metrics \textit{Accuracy} (\textit{exact\_match}), \textit{Time} (\textit{inference\_runtime}) and \textit{Length} (\textit{num\_bytes}). Pairwise majority yields a transitive ranking in 3825 of these 4731 configurations. The following experiments use these 3825 profiles and the alphabetical tie-breaking rule.

\paragraph{Changes in runtime.}
For those 3825 cases, we test whether the outcome is unchanged under every independent runtime change within the given range ($\pm1\%$, $\pm5\%$ or $\pm10\%$ of the model's original runtime on the given subject), while holding the values of accuracy and output length fixed (see Table~\ref{tab:runtime_changes}). To investigate this for every permitted change, we compare the endpoints of the runtime ranges for each pair of models on which the metrics \textit{Accuracy} and \textit{Length} disagree.

\begin{table}[h]
\centering
\caption{Percentage of the 3825 profiles for which the original winner model or the entire ranking is unchanged under the given runtime change.}
\label{tab:runtime_changes}
\vspace{4pt}
\begin{tabular}{@{}lrr@{}}
\toprule
Change in runtime &
\shortstack[r]{Original winner model\\beats every other model} &
\shortstack[r]{Entire ranking\\is unchanged} \\
\midrule
$\pm1\%$  & 98.20 & 95.74 \\
$\pm5\%$  & 92.39 & 80.86 \\
$\pm10\%$ & 85.54 & 67.27 \\
\bottomrule
\end{tabular}
\end{table}

\paragraph{Paired bootstrap samples of questions.}

We draw 10000 paired bootstrap samples of MMLU questions for each subject. Each sample consists of the original questions that are sampled with replacement. Within each sample, the same questions are used for all models and metrics. We recalculate either \textit{Time} alone, keeping \textit{Accuracy} and \textit{Length} fixed or all three metrics together (see Table~\ref{tab:question_bootstrap}).

\begin{table}[h]
\centering
\caption{Percentage of paired bootstrap samples in which the original winner model or the entire ranking is unchanged, averaged across the 3825 profiles.}
\label{tab:question_bootstrap}
\vspace{4pt}
\begin{tabular}{@{}p{0.40\linewidth}rr@{}}
\toprule
Metrics &
\shortstack[r]{Original winner model\\beats every other model} &
\shortstack[r]{Entire ranking\\is unchanged} \\
\midrule
\textit{Time} only
& 98.2 & 95.9 \\
\textit{Accuracy}, \textit{Time} and \textit{Length}
& 87.1 & 75.6 \\
\bottomrule
\end{tabular}
\end{table}

When all three metrics are recalculated, pairwise majority is transitive in $94.4\%$ of samples, averaged across the 3825 profiles. Marginal changes in runtime thus rarely overturn the selected winner model in our experiments.

\subsection{Metric Weights}

The procedure of majority voting can be adapted to include different metric weights. An application that is more focused on accuracy can assign greater importance to this criterion. We can repeat a metric's ranking and thus incorporate integer weights. This preserves single-peakedness and group separability, so their transitivity guarantees apply under the corresponding theorem assumptions.

\subsection{Aggregation across Datasets}\label{app:across_datasets}
For aggregation across datasets, we report the commonality sharing aggregate only for a small subset of configurations for which the majority rankings are coherent on all $57$ MMLU subjects. Figure~\ref{commonality_sharing_rest} is therefore an illustration of the final aggregation step once the relevant structural condition has been verified. The commonality sharing ranking can be understood as a representative ranking in the whole benchmark suite. In statistical terms, it is a generalization of the concept of a median to the case of datasets, where each data point is a ranking. The deepest relation is computed by calculating the Tukey depth of all rankings using the R package ddandrda \citep{ddandrda}. 

\begin{figure}[t]
\centering
\includegraphics[width=0.3\textwidth]{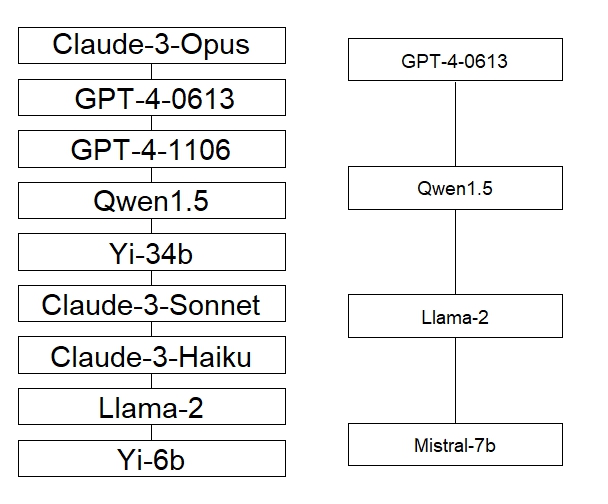}
\vspace{0.5em}

\includegraphics[width=0.4\textwidth]{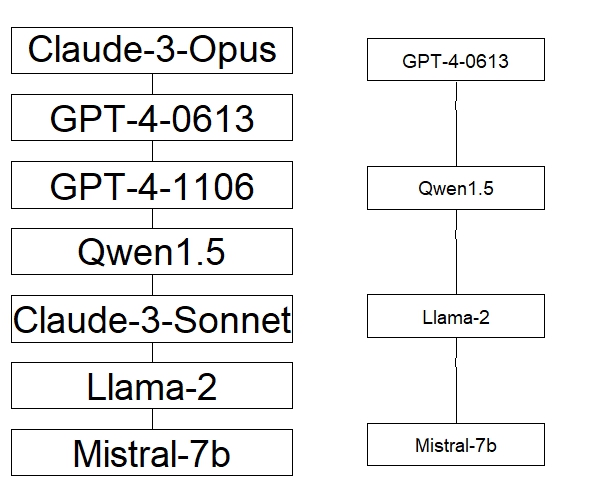}
\hfill \includegraphics[width=0.4\textwidth]{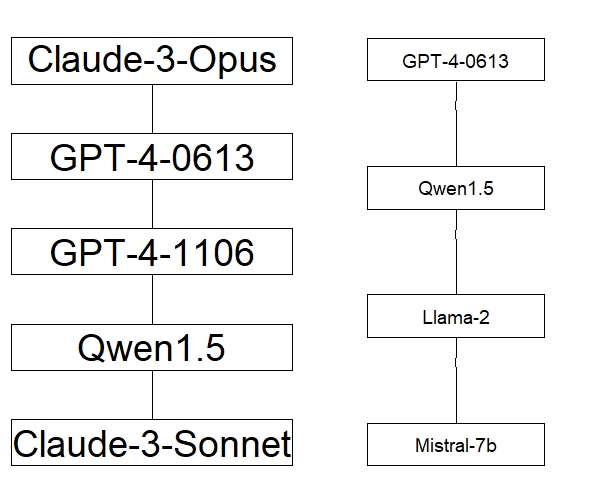}
\caption{Aggregated rankings across all $57$ datasets of HELM MMLU according to the commonality sharing rule (cf., Section~\ref{section_aggregation_across_datasets}).  Top: single-peaked preferences for $\mathcal{A}_1$ with the accuracy metrics $\{$\textit{exact\_match}, \textit{prefix\_exact\_match}, \textit{quasi\_exact\_match}$\}$ and for $\mathcal{A}_3$ with the mixed metrics $\{$\textit{exact\_match}, \textit{quasi\_exact\_match}, \textit{inference\_runtime}$\}$.  Bottom left: group separability for $\mathcal{A}_2$ with the same accuracy metrics and for $\mathcal{A}_3$ with the same mixed metric triple. Bottom right: distance-restrictedness to degree $1$ for $\mathcal{A}_4$ and $\mathcal{A}_3$, both with the accuracy metric triple.}\label{commonality_sharing_rest}
\end{figure}

\section{Experiments on Restricted Preference Domains: PMLB and OpenML}\label{app:exp_other}

For our verifications of restricted preference domains on PMLB and OpenML benchmarks we use the final results from \citet{gsd} and then proceed exactly as in the tests on HELM MMLU. 

\subsection{PMLB}
The benchmark suite includes six classifiers (cre, J48, glmnet, knn, ranger, svmRadial) evaluated on three metrics (predictive accuracy, robustness to feature perturbation and robustness to label perturbation) over 63 datasets. We find that seven out of 63 datasets are single-peaked. For group separability, we observe that 11 out of 63 datasets are group-separable. The results remain very similar after switching to the reverse alphabetical tie-breaking. Since the metrics measure related aspects of one latent concept (robust accuracy), it is plausible that we observe some single-peaked and group-separable datasets; however, the metrics are still different enough so that there is non-trivial disagreement between rankings which leads to multiple violations of our structural assumptions.

\subsection{OpenML}
The suite contains seven classifiers (classif.glmnet, classif.kknn, classif.multinom, classif.ranger, classif.rpart, classif.svm, classif.xgboost) evaluated on 80 datasets and three metrics (predictive accuracy, computation time on training data and computation time on test data). We find that none of 80 datasets is single-peaked, and four out of 80 datasets are group-separable. The results are unchanged under the reverse alphabetical tie-breaking. The combination of accuracy with training and test runtime induces a strong trade-off, meaning that different metrics disagree on many pairwise comparisons, which implies less structured preference profiles.

\section{Experiments on Model Comparison and Selection}
\label{app:selection}

\subsection{Metric Sets and Competing Objectives}\label{app:metrics_obj}
In Table~\ref{tab:metric_groups}, we group the metric subsets of size three by the number of accuracy and efficiency metrics. There are four subsets with three accuracy metrics, twelve with two accuracy metrics and one efficiency metric as well as four subsets with one accuracy metric and two efficiency metrics. Across the 83 model sets and 57 subjects, these groups include 18924, 56772 and 18924 profiles respectively. The table shows the percentage of profiles that fulfill each property under the alphabetical tie-breaking.

As an additional test, using the three metrics \textit{Accuracy} (\textit{exact\_match}), \textit{Time} (\textit{inference\_runtime}) and \textit{Length} (\textit{num\_bytes}), we investigate what happens if each of the three metrics prefers a different model. Across the 83 model sets and 57 MMLU subjects, 760 profiles have a uniquely best model for each of the three metrics (where the three best models are different). Among those, 256 ($33.7\%$) profiles satisfy at least one of the sufficient conditions. Therefore, the conditions can be fulfilled even when accuracy and efficiency metrics favour different models, and in these cases, majority voting still yields a complete and transitive ranking of models.

Table~\ref{tab:human_aging} demonstrates the metric values for the example in Section~\ref{singlepeaked}. Although each metric prefers a different model, the profile is still single-peaked and pairwise majority ranks GPT-4-0613 above GPT-3.5-Turbo-0613 above GPT-4-1106-Preview. Moreover, using the procedure in Appendix~\ref{app:sensitivity}, we find that both the winner and the entire ranking remain unchanged under every independent $\pm5\%$ runtime change, with the criteria \textit{Accuracy} and \textit{Length} kept fixed.

\begin{table}[h]
\centering
\caption{Metric values on \textit{Human Aging} for model set $\mathcal{A}_{21}$. Higher \textit{Accuracy} and lower \textit{Time} and \textit{Length} are preferable.}
\label{tab:human_aging}
\vspace{4pt}
\begin{tabular}{@{}lrrr@{}}
\toprule
Model &
Accuracy (\%) &
Time (s) &
Length (bytes) \\
\midrule
GPT-4-0613          & 83.408 & 0.405815 & 1.000000 \\
GPT-4-1106-Preview  & 85.650 & 0.472131 & 1.026906 \\
GPT-3.5-Turbo-0613  & 76.233 & 0.379983 & 1.013453 \\
\bottomrule
\end{tabular}
\end{table}

\subsection{Model Selection When the Conditions Fail}\label{app:selection_possibilities}

We test the possibilities described in Section~\ref{disc} on all three benchmarks under alphabetical tie-breaking.

\begin{table}[h]
\centering
\caption{Structural conditions, complete majority rankings and selection of one model. Note that the last column includes the cases with a transitive ranking.}
\vspace{4pt}
\begin{tabular}{@{}lrrr@{}}
\toprule
Benchmark &
\shortstack[r]{At least one\\condition holds} &
\shortstack[r]{Complete and transitive\\majority ranking} &
\shortstack[r]{Model beats\\every other model} \\
\midrule
HELM MMLU & 3081 / 4731 & 3825 / 4731 & 3965 / 4731 \\
PMLB      & 17 / 63       & 51 / 63       & 60 / 63       \\
OpenML    & 4 / 80        & 40 / 80       & 61 / 80       \\
\bottomrule
\end{tabular}
\end{table}

On HELM MMLU, this test considers \textit{Accuracy} (\textit{exact\_match}), \textit{Time} (\textit{inference\_runtime}) and \textit{Length} (\textit{num\_bytes}). PMLB uses accuracy and two robustness metrics, while OpenML includes accuracy and training and test runtime.

The results provide concrete guidance for model comparison and selection even when the structural conditions from Section~\ref{analysis} fail. They also distinguish the requirements for choosing one model from those for reporting a complete benchmark ranking.

\section{Experiments: Transitivity and Stability}\label{app:expmotiv}

\subsection{Condorcet Cycles}\label{app:cycles} 
We preprocess the HELM MMLU data exactly as in Section~\ref{setupexp}. We align metrics so that larger values mean better performance by flipping the sign for ``lower-is-better'' metrics such as \textit{Time} and \textit{Length}. Then, for each subject, we search over all triples of metrics in the fixed set $\{$\textit{exact\_match}, \textit{inference\_runtime}, \textit{num\_bytes}, \textit{logprob}, \textit{perplexity}, \textit{num\_output\_tokens}$\}$. Within a metric triple, we restrict to models with complete observations and define pairwise majority comparisons. For each pair of models, each metric ``votes'' only when the two aligned scores differ by more than a small tolerance; ties and near-tie situations are treated as abstentions. Then, we search for cycles i.e., triples of models $A,A^{'},A^{''}$ such that $A$ is preferred to $A^{'}$, $A^{'}$ is preferred to $A^{''}$, and $A^{''}$ is preferred to $A$. A subject is counted as cyclic if at least one metric triple yields such a cycle. For each subject, we keep the most robust witnessed cycle and assign it a buffer. For each of the three pairwise majority wins in the cycle, we take the smallest score gap among the metrics supporting that win so that the cycle’s buffer is the minimum of these three values. Larger buffers therefore correspond to cycles that persist under small perturbations of the metric values. Under this procedure, we find at least one Condorcet cycle in 51 of 57 MMLU subjects. This suggests that cyclic preferences are widespread once accuracy and efficiency metrics are combined.

The full model names for the cycle presented in Table~\ref{table:motiv_examples} are GPT-3.5-Turbo-0613, GPT-4-1106-Preview and Qwen1.5-14b. Note that in the experiments in Section~\ref{setupexp} (also covered in Appendix \ref{app:exphelm}) the model name Qwen1.5 is used as the abbreviation for the model Qwen1.5-72b.

\subsection{Instability to Changes in the Model Set}\label{app:instability}

For each MMLU subject, we fix the metric set $\{$\textit{exact\_match}, \textit{inference\_runtime}, \textit{num\_bytes}$\}$ (flipping the sign for ``lower-is-better'' metrics as above). For each subject, we restrict to models with complete observations and focus only on the top 15 models by the metric \textit{exact\_match}. For each subject, we compute the overall ranking by averaging the metrics’ ranks, then add one additional model and recompute the ranking. We record whether any pairwise order among the original 15 models flips under this change. We find such flips in 44 out of 57 MMLU subjects.

The full model names for all models mentioned in Table~\ref{table:motiv_examples} are as follows: GPT-4-0613, GPT-4-1106-Preview, Claude-3-Opus-20240229, Qwen1.5-72b, GPT-3.5-Turbo-0613, Google-Text-Bison@001 and Llama-2-13b. Note that in the experiments in Section~\ref{setupexp} (also covered in Appendix \ref{app:exphelm}) the model name Llama-2 is used as the abbreviation for the model Llama-2-70b.

\subsection{Transitivity of Pairwise Majority}
\label{app:transitivity}
We test 83 model sets and 42 metric subsets (3486 combinations of model and metric sets and 198702 profiles on individual subjects) for transitivity of the induced pairwise majority relation on HELM MMLU. We use alphabetical tie-breaking and ten repetitions of random tie-breaking independently for each metric in each profile and compare these results with uniformly random profiles with the same numbers of models and metrics.

\begin{table}[h]
\centering
\caption{Percentage of profiles with a transitive weak pairwise majority relation. The last row shows the percentage of configurations that are transitive on all 57 subjects. The results for random tie-breaking are means over ten repetitions.}
\label{tab:additional_transitivity}
\vspace{4pt}
\small
\begin{tabular}{@{}p{0.40\linewidth}rrr@{}}
\toprule
Setting &
\shortstack[r]{HELM\\MMLU} &
\shortstack[r]{HELM MMLU,\\ random tie-breaking} &
\shortstack[r]{Uniformly\\random} \\
\midrule
3 models, 3 metrics
& 98.343 & 98.426 & 94.444 \\
5 models, 3 metrics
& 94.025 & 94.762 & 67.569 \\
9 models, 3 metrics
& 80.000 & 80.825 & 14.198 \\
9 models, 4 metrics
& 58.480 & 57.708 & 0.033 \\
All 198702 profiles on individual subjects
& 93.651 & 93.542 & 66.638 \\
Configurations transitive on all 57 subjects
& 78.371 & 75.175 & 0.872 \\
\bottomrule
\end{tabular}
\end{table}

Note that in the uniformly random baseline, the 57 subject profiles are independent by construction. The results imply that transitivity occurs more often on HELM MMLU as compared to the uniformly random case.

\section{Licenses}\label{app:license}
We only use existing benchmark outputs and software. HELM is credited through \cite{liang2023} and used through its public framework and results; the HELM code is released under the Apache License 2.0. MMLU is credited through \cite{Hendrycks2021} and its public repository is released under the MIT License. For the additional PMLB and OpenML checks, we use the final results from \citet{gsd}. The R software is licensed under GPL-2 or GPL-3; the main R packages used in our scripts are \texttt{jsonlite}, \texttt{dplyr}, \texttt{purrr} and \texttt{tibble} under MIT, \texttt{digest}, \texttt{lme4}, \texttt{future.apply} and \texttt{relations} under GPL-compatible licenses, \texttt{future} and \texttt{ddandrda} under LGPL-compatible licenses and \texttt{Rgraphviz} under EPL.

\end{document}